\documentclass{article}

\usepackage{graphicx} % Required for inserting images

\usepackage{arxiv}

\usepackage[utf8]{inputenc} % allow utf-8 input
\usepackage[T1]{fontenc}    % use 8-bit T1 fonts
\usepackage{hyperref}       % hyperlinks
\usepackage{url}            % simple URL typesetting
\usepackage{booktabs}       % professional-quality tables
\usepackage{nicefrac}       % compact symbols for 1/2, etc.
\usepackage{microtype}      % microtypography

\usepackage{amsmath}
\usepackage{amsfonts}
\usepackage{amsthm}
\usepackage{float}
\usepackage{geometry}

\usepackage{multirow}

\usepackage{subcaption}
\usepackage{bbm}
\usepackage{xcolor}
\newcommand\hlbreakable[1]{#1}

\usepackage[square,numbers]{natbib}
\usepackage{theoremref}
\usepackage{graphicx}
\usepackage{natbib}
\usepackage{doi}

\title{Entropy-based Guidance of Deep Neural Networks for Accelerated Convergence and Improved Performance}

\author{ \hspace{1mm}Mackenzie J.~Meni\\
        NEural TransmissionS (NETS) Lab\\
	Florida Institute of Technology\\
	\href{mailto:mmeni2021@my.fit.edu}{\texttt{mmeni2021@my.fit.edu}} \\
	\And
	\href{https://orcid.org/0000-0002-5524-629X}{\includegraphics[scale=0.06]{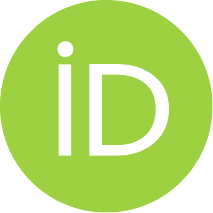}\hspace{1mm}Ryan T.~White} \\
        NEural TransmissionS (NETS) Lab\\
	Florida Institute of Technology\\
	\href{mailto:rwhite@fit.edu}{\texttt{rwhite@fit.edu}} \\
        \And
	\hspace{1mm}Michael L.~Mayo\\
        Environmental Laboratory \\
        U.S. Army Engineer Research and Development Center (ERDC) \\
        \href{mailto:Michael.L.Mayo@erdc.dren.mil}{\texttt{Michael.L.Mayo@erdc.dren.mil}} \\
        \And
	\hspace{1mm}Kevin R.~Pilkiewicz \\
        Environmental Laboratory \\
        U.S. Army Engineer Research and Development Center (ERDC) \\
        \href{mailto:Kevin.R.Pilkiewicz@usace.army.mil}{\texttt{Kevin.R.Pilkiewicz@usace.army.mil}} \\
}

\renewcommand{\shorttitle}{Entropy-Based Guidance of Deep Neural Networks}

\hypersetup{
pdftitle={A template for the arxiv style},
pdfsubject={q-bio.NC, q-bio.QM},
pdfauthor={David S.~Hippocampus, Elias D.~Striatum},
pdfkeywords={First keyword, Second keyword, More},
}

\date{}

\usepackage{arydshln}
\theoremstyle{plain}
\newtheorem{theorem}{Theorem}

\newtheorem{corollary}[theorem]{Corollary}
\theoremstyle{definition}
\newtheorem{definition}[theorem]{Definition}

\newtheorem{example}[theorem]{Example}

\begin{document}

\maketitle

\begin{abstract}
    Neural networks have dramatically increased our capacity to learn from large, high-dimensional datasets across innumerable disciplines. However, their decisions are not easily interpretable, their computational costs are high, and building and training them are not straightforward processes. To add structure to these efforts, we derive new mathematical results to efficiently measure the changes in entropy as fully-connected and convolutional neural networks process data. By measuring the change in entropy as networks process data effectively, patterns critical to a well-performing network can be visualized and identified. Entropy-based loss terms are developed to improve dense and convolutional model accuracy and efficiency by promoting the ideal entropy patterns. Experiments in image compression, image classification, and image segmentation on benchmark datasets demonstrate these losses guide neural networks to learn rich latent data representations in fewer dimensions, converge in fewer training epochs, and achieve higher accuracy.
    \end{abstract}

\section{Introduction}

In the past 15 years, neural networks have revolutionized our capabilities in computer vision using convolutional neural networks (CNNs), for example, seminal works by \citet{lecun_gradient-based_1998} and \citet{krizhevsky_imagenet_2012}, and vision Transformers from \citet{dosovitskiy_image_2021}, natural language processing (e.g., by \citet{radford_improving_2018, ouyang_training_2022}) with Transformers from \citet{vaswani2017attention}, synthetic data generation with generative adversarial networks (GANs) from \citet{goodfellow_generative_2014} and diffusion models introduced by \citet{ho_denoising_2020}, and innumerable other domains.

Neural networks have become essential tools in numerous fields, ranging from healthcare to national security, due to their exceptional predictive power and versatility. However, the opaque nature of these models poses significant challenges, particularly in high-stakes environments where understanding and trust in how decisions are made are crucial. Developing more interpretable neural network models could revolutionize their use, providing clear, actionable insights into why and how certain decisions are made. This clarity could enhance the efficiency of the models, improve collaboration with domain experts, and deepen our understanding of critical features and their impacts.

Moreover, the complexity of neural networks complicates their deployment and explanation. These models often engage with high-dimensional data sets, require the optimization of millions, if not billions, of parameters, and operate through expansive architectures that process data in intricate ways during both training and inference phases. This complexity not only demands substantial computational resources but also makes it exceedingly challenging to elucidate the models’ decision-making processes. Enhancing the interpretability of neural networks could lead to more trustworthy applications, particularly in fields where the stakes are exceptionally high and decisions have profound implications on human lives. By addressing these challenges, we could unlock a new level of utility and reliability in neural network applications, ensuring they meet the rigorous demands of critical use cases.

Information theory demonstrates the potential to provide probabilistic explanations of how neural networks process data and make decisions. This can allow us to interpret model decisions, get more from smaller neural architectures to reduce data processing complexity, encourage better generalization, and optimize them more efficiently. In addition, analysis of patterns in information flow through neural networks provides a pathway to understand what trends during training and/or inference promote good performance. This leads to information-theoretic metrics to consider and an array of hyperparameters to encourage ideal behavior.

Prior works have exploited entropy and mutual information in deep learning for preprocessing data, training neural networks, regularizing the training, and interpreting their decisions. These have mostly focused estimate information theoretic quantities like entropy of neural network outputs or latent distributions, or mutual information between latent representations and other representations or outputs to enable pursing downstream benefits like increased performance, efficiency, or interpretability.  However, estimating these quantities directly for latent representations yields many challenges due to high dimensionality and resulting unreliable estimations and/or high computational costs. This hinders their practicality in many environments.

This article seeks to sidestep these challenges by only deriving explicit formulas for the the change in entropy from one hidden layer to the next within a neural network. We first derive novel probabilistic results for measuring entropy propagation through critical structures within neural networks: fully-connected and 2D convolutional layers. These formulas are used to create entropy-based loss functions. We then use these results to guide the training of neural networks, leading to better convergence times and model performance. This yields a practical entropy-based approach that yield numerous benefits across diverse computer vision tasks and neural architectures.

The main contributions of this work include:

\begin{itemize}
    \item Development of tractable formulas for entropy propagation through dense and convolutional layers and corresponding novel entropy-based loss terms to enable entropy-based guidance of neural network training. 
    \item An empirical analysis of entropy propagation patterns through well-trained supervised neural networks, providing an understanding of ideal information flows, and optimal entropy-based guidance strategy.
    \item Experimental demonstrations that entropy-based loss terms speed up convergence in image compression on MNIST and CIFAR10, and improve image classification performance on CIFAR10 with statistical validation.
    \item \hlbreakable{Large-scale experimental demonstrations of improved performance and convergence on large models, such as VGG-16, ResNet, U-Net, with larger, more realistic benchmark datasets.}
\end{itemize}

The remainder of the paper is structured as follows. Section 2 outlines some related works on information-theoretic learning. In section 3, we derive probabilistic results for the propagation of entropy through dense and convolutional neural network layers. Section 4 adapts these results into novel entropy-based loss terms for guiding neural network training. Section 5 discusses experimental results in applying the entropy-based loss terms for training neural networks for image reconstruction, classification, \hlbreakable{and segmentation} tasks on standard benchmark image datasets. Section 6 provides concluding remarks on the article's findings, contributions, and implications for future work.

\section{Related Work}

\textbf{Information Theoretic Learning.}  \citet{principe_learning_2000} address the idea of information-theoretic learning as a way of decoding how machines learn from data. In this work, they discuss a possible framework utilizing R\'{e}nyi’s quadratic entropy, \hlbreakable{computed in the discrete case as}
\begin{align}
    H_2(X)=-\log\left(\sum\limits_i f(x_i)^2\right)
\end{align}
\hlbreakable{for a random variable $X$ with probability mass function $f$. This is used to develop an $O(N^2)$ entropy estimator, which is used train linear or nonlinear mappers with the goal of entropy minimization or maximization.}

\citet{erdogmus_online_2003} later propose a stochastic entropy estimator based on Parzen window estimates of the probability density function (pdf). This enables a stochastic gradient that allows entropy optimization through gradient-based optimization. These are foundational works enabling potential entropy manipulation, although their computational complexity hindered the practical applicability in real-time or resource-constrained environments.

\hlbreakable{\textbf{Entropy Estimation.} Estimating the entropy of outputs or latent representations of large deep learning models is challenging due to the complexity of models and high dimensionality of their hidden representations.}

\hlbreakable{Many classical methods exist for estimating entropy, including simple plug-in estimates, and nearest neighbors estimates. See \citet{beirlant1997nonparametric} for an overview. However, they have limitations preventing use for our problem. Plug-in estimates such as \citet{dmitriev1974estimation} require density estimation and/or numerical integration, which is challenging in high dimensions. Histogram-based estimation from \citet{gyorfi1987density} reduce computational costs, but not for high dimensional cases. Sample spacing estimates as in \citet{tarasenko1968evaluation} are limited to univariate entropy. Nearest neighbors techniques presented early by \citet{bickel1983sums} struggle in high dimensions due to a lack of sufficient density of sampled points.}

\hlbreakable{Notably, \citet{gabrie2018entropy} tackle this problem by using hueristics from statistical physics to estimate entropies and mutual information between layers of a class of feedforward neural networks, including nonlinearities, arbitrary weidth and depth, and correlated input data. It does, however, require some constraints on the weight matrices not realistic in practice. \citet{dong2023understanding} use the kernelized R\'{e}nyi's entropy as a more tractable way to measure information theoretic quantities in high-dimensional settings in deep learning. The authors derive rigorous and tight generalization error bounds and demonstrate empirical agreement with experiment. These bounds, however, which are not of use for regularization or entropy-based guidance, as our work does.}

\textbf{Information Theoretic Losses.} More recently, many additional methods have been suggested in the utilization of information theory to decode algorithmic decisions and processes. Specific to deep learning models, information-theoretic loss functions have been a promising approach. Such loss functions include cross-entropy, F-divergence, mutual information (MI) losses, Kullback–Leibler (KL) divergence, and others derived from information theory.%, as outlined in the survey by \citet{yu_information-theoretic_2021}.

\hlbreakable{Relative entropy (i.e. KL divergence) is used by \citet{QM_kldiv_2020} as a base loss function the penalizes deviations of predicted distributions from the ground truth labels across all training examples, in the discrete case as:}
\begin{align}
    \mathcal{L}(f,\hat{f})=-\sum\limits_{i=1}^d \sigma\left(f_i\right)\log\left(\frac{\sigma(\hat{f}_i)}{\sigma(f_i)}\right)
\end{align}
\hlbreakable{for a single ground truth target vector $f$ and prediction $\hat{f}$, where and $\sigma$ is the softmax function. This enables them to calibrate rescaled output data for the prediction of extreme events and anomalous features, particularly extreme events. This approach was extended by \citet{rudy2023output} to an analogous loss with a softmax-like operator applied to an entire batch of data examples. This enables a better matching between predictions and the sampling distribution and primes the model to predict extreme events, although it requires probability density estimation. For quite a different use case, \citet{Liu_RDE_2020} use a relative entropy loss while approximating densities by histograms, setting state-of-the-art marks in several benchmark image retrieval tasks.}

In another direction, InfoMax, short for Information Maximization, is a learning principle that aims to maximize the mutual information between the input data and some learned representation, typically a latent variable, as explored by  \citet{linsker_application_1988}. The core idea is to design models that capture and retain as much information as possible from the input data in the learned representations, with the assumption that this information will be useful for downstream tasks like classification or generation. \citet{hjelm_learning_2019} introduced Deep InfoMax (DIM) to estimate and maximize the MI between the input data and the latent representations. However, this was shown to result in excessive and noisy information by \citet{yu_information-theoretic_2021}. While DIM may lead to representations that capture more information from the input data, it might not guarantee that the learned representations are interpretable or meaningful to humans. 

\textbf{Information Bottleneck.} One suggested method to mitigate such issues is the Information Bottleneck (IB) principle proposed by \citet{tishby_information_1999}. The IB is an information-theoretic framework for learning representations in a way that balances the amount of information captured from the input data with its relevance to a target output. The main focus of IB is to extract a compressed and informative representation from the input data retaining only the essential information needed to predict the target output, as applied to deep networks by \citet{yu_deep_2021}. For example, \citet{xu_q-detr_2023} recently used a ``teacher'' object detection Transformer model and distill knowledge into an efficient ``student'' quantized Transformer model. They use an IB-inspired idea to take alternating training steps to minimize the entropy of the student's latent representations conditioned on the teacher's while also maximizing the (unconditional) entropy of the student queries. \hlbreakable{While effective in such cases, IB assumes that the input-output relationship in the data can be accurately captured by a single target variable. This assumption might not hold true for all real-world scenarios where complex relationships exist between multiple variables.}

\hlbreakable{\textbf{Feature Engineering and Understanding.} Further works in information theoretic learning focus on feature engineering and understanding. For example, \citet{lhermitte2022deep} and \citet{hayashi2023image} explore entropy-based preprocessing and feature extraction from RGB images for downstream tasks. \citet{Gajowniczek_GEL_2020} use R\'{e}nyi's entropy for sensitivity analysis and to evaluate feature importance in deep learning. \citet{amanova2024finding} locate input features that reduce entropy of neural network predictions, which enables targeted perturbations that explain the importance of different features. \citet{finnegan2017maximum} uses maximum entropy methods to extract learned features from a trained neural network. In analyzing DNA sequences, they can identify and extract biologically significant learned features. Further, \citet{ozdenizci2021stochastic} introduces a training procedure using stochastic estimates of mutual information gradients for feature selection. The authors train a corresponding dimension-reducing neural network that projects high-dimensional features onto a lower dimensional feature space where the data representations have maximum mutual information with their class labels.}

\textbf{Generative AI.} Information theory has also been incorporated into generative artificial intelligence (AI). \citet{vincent_stacked_2010} show minimizing an autoencoder's reconstruction error coincides with maximizing a lower bound on information, similar to InfoMax. They pair this with a denoising criterion to learn latent representations from which the decoder then generates synthetic data. InfoGAN, introduced by \citet{chen_2016_infogan}, enhances the original GAN loss with a term that maximizes a lower bound for the mutual information between a small subset of the latent variables and the observation. Unlike standard GANs, this provides interpretable latent variables that can be manipulated to generate synthetic data with specific properties (e.g., synthetic MNIST digits from specified classes or with specific rotation or boldness). These efforts demonstrate information-theoretic losses can enable rich, interpretable latent representations that avoid mode collapse for generative neural networks, suggesting their application in supervised domains. Variational autoencoders in \citet{kingma_auto-encoding_2013} take a Bayesian view and considers prior distribution and noisy posterior distribution for the latent representations. They create a loss that minimizes reconstruction error while encouraging high entropy of the posterior and penalizing cross-entropy between the distributions. This results in latent representations that account for greater variation in the inputs, avoiding mode collapse. 

\hlbreakable{\textbf{Entropy Regularization in Deep Reinforcement Learning.} In reinforcement learning, an agent takes initially random actions from an action space and is trained to take favorable actions through a reward and penalty structure. Modern deep reinforcement learning makes use of neural networks infer decisions based on input data. Entropy regularization is commonly used (e.g. in \citet{mnih2016asynchronous, haarnoja2018soft}) to control the exploration of action spaces--high entropy means the agent has more freedom to explore, low entropy means the agent is more constrained. Additionally, work by \citet{han2023entropy} demonstrates entropy regularization improves performance when used in optimizing policies--the agent's way of behaving when presented with a state and an environment. \citet{hao2022entropy} demonstrated such entropy-based control even for challenging multi-agent reinforcement learning problems where there are multiple agents simultaneously learning. Their automatic regularization makes the multi-agent learning more stable while ensuring exploration of the action space by the agents.}

\textbf{RLHF and ChatGPT.} Reinforcement learning with human feedback (RLHF) allows reward learning where the rewards are defined by human judgment and has been used in recent years to fine-tune language models to generate better synthetic text. Recent work by \citet{stiennon_learning_2020} use human preference labels on text generated by GPT-2 to develop a reward model and train a policy that generates higher-quality text as judged by human preferences. In the reward model, an entropy bonus is implemented to encourage the policy to explore the parameter space widely and avoid collapsing modes. The InstructGPT model of \citet{ouyang_training_2022} incorporated this approach into GPT-3.5. Here, information-theoretic loss terms permit meaningfully constrained explorations of the latent space, though the scale of these models (13-175B parameters) results in an extremely high-dimensional latent space that is difficult to interpret.

\section{Probabilistic Results}

In this section, we state some definitions and results from the information-theoretic literature. We then establish several new results that allow the application of these ideas to fully-connected feedforward neural networks (i.e. multilayer perceptrons or MLPs) and CNNs. These results enable tracking the evolution of entropy of data as it passes through a neural network.

First, we provide a definition of entropy for use in analyzing neural networks.

\begin{definition}
Shannon's (joint) differential entropy of a continuous random variable $X$ valued in $\mathbb{R}^d$ with joint probability density function (pdf) $f$ is
\begin{align}
H(X)=\mathbb{E}\left[-\log f(X)\right].
\end{align}
\end{definition}

Thus, $H$ provides a metric for the average information content of a random variable, encoding the extent of its accessible states. Unless otherwise noted, we will use the word \textit{entropy} to refer to joint differentiable entropy, and we use natural logarithms throughout the article (i.e., the entropy being expressed in units of nats). While a pdf, $f$, of the random variable $X$ is assumed to exist, it is not assumed to be known. We almost never know the pdf of high-dimensional datasets for computer vision, natural language processing (NLP), or other domains where deep neural networks are effective.

We will represent dense and 2D convolutional neural network layers as matrix-vector products with invertible, constant matrices multiplying random input data or latent representations of that data. Next we present a formula for the entropy propagation of a matrix-vector product $WX$, where $X$ is a random vector and $W$ is a constant matrix. This permits efficient estimation of the change in entropy as data propagates through the dense and convolutional layers of a neural network.

\begin{theorem}
\thlabel{entropy_matrix_mult}(\citet{Cover2006}, Corollary to Theorem 8.6.4) Let $X$ be a random variable valued in $\mathbb{R}^d$ and constant $W\in\mathbb{R}^{d\times d}$. If $W$ is invertible, then the entropy of $WX$ is
\begin{align}
H(WX)=H(X)+\log\left(\left|\det W\right|\right).
\end{align}
\end{theorem}

However, weight matrices in dense layers may not be invertible or even square, so it is unclear how to use this result to measure entropy propagation. Even worse, convolutions are typically not considered a matrix-vector product at all. The next two subsections offer remedies to these issues.

\subsection{Dense Layers}

The following theorem is a novel result that computes the entropy of pre-activation values $W'X$ as the (unknown) input entropy of $X$ plus an easily computable term.
\begin{theorem}\thlabel{DenseEntropyProp}
    Suppose $X:\Omega\to\mathbb{R}^{d\times N}$ ($N>d$) is a random matrix and $W\in\mathbb{R}^{\min(d,m)\times \min(d,m)}$ is an invertible matrix. Then,
    \begin{align}
        H(W'X)=H(X)+\log\left(\left|\det W\right|\right)
    \end{align}
    where
    \begin{align}\label{SquarifyMatrix}
    W'=\begin{cases}
    \begin{pmatrix}
        W & W_{N\times (d-N)} \\
        0 & I_{d-N}
    \end{pmatrix}, & \text{if }N<d\\
    W, & \text{if }N=d\\
    \begin{pmatrix}
        W & 0\\
        W_{(N-d)\times d} & I_{N-d}
    \end{pmatrix},&\text{if }N>d
    \end{cases}
    \end{align}
\end{theorem}
\begin{proof}
We refer to $W$ as the square part of $W'$. Note $W'$ is block upper diagonal, square, or block lower diagonal, depending on the input dimension $d$ and output dimension $N$. In any of the three cases,
\begin{align*}
\det\left(W'\right) = \det\left(W\right).
\end{align*}

Since $W$ is invertible, $\det\left(W'\right)=\det\left(W\right)>0$, so $W'$ is invertible. Then, \thref{entropy_matrix_mult} implies
\begin{align*}
    H(W'X)=H(X)+\log\left(\left|\det W'\right|\right)=H(X)+\log\left(\left|\det W\right|\right).
\end{align*}
\end{proof}

This formula clarifies how entropy propagates from input to pre-activations within dense layers.

\begin{example}
If a weight matrix has more columns than rows ($3=N<d=5$ in this case), \thref{DenseEntropyProp} suggests modifying it as:
\begin{align*}
\text{Weight Matrix}=\begin{pmatrix}
    3 & 0 & 9 & -3 & 4\\
    1 & 5 & -1 & 4 & 2\\
    0 & 4 & -2 & 1 & 5\\
\end{pmatrix}
    \to \left(
        \begin{array}{ccc:cc}
            3 & 0 & 9 & -3 & 4\\
            1 & 5 & -1 & 4 & 2\\
            0 & 4 & -2 & 1 & 5\\ \hdashline
            0 & 0 & 0 & 1 & 0\\
            0 & 0 & 0 & 0 & 1
        \end{array}
    \right)=\begin{pmatrix}
        W & W_{3\times 2}\\
        0 & I_{2}
    \end{pmatrix}=W',
\end{align*}
where $W$ is the square part of the weight matrix, $W_{3\times 2}$ is the remaining portion of the weight matrix, and $I_2$ is a $2\times2$ identity block.

A similar manipulation occurs when there are more rows than columns, with a block of zeros placed on the upper right and identity block placed on the lower right.
\end{example}

\subsection{2D Convolutions}

Consider an $p$-by-$q$ convolutional filter $C$. If $C$ convolves with strides $1\times 1$ over an input image $X\in\mathbb{R}^{l\times w}$, then an output pre-activation map $Z = C*X$ is defined as
\begin{align}
Z_{ij}=\sum\limits_{k=0}^{p}\sum\limits_{l=0}^{q}C_{kl}X_{i+k,j+l}
\end{align}
for $i=0,1,...,l-p$ and $j=0,1,..., w-q$. While this is typically envisioned as the 2D convolutional filter $C$ scanning over the image $X$ to extract features denoted, the operation can alternatively be represented as a matrix-vector product. We will establish $Z_F=C_MX_F$, where $X_F$ and $Z_F$ are flattened versions of $X$ and $Z$ and $C_M$ has a special structure constructed below.

First, we transpose the rows $x_i^T$ of matrix $X\in\mathbb{R}^{l\times w}$ into column vectors and then concatenate them into a single column vector $X_F\in\mathbb{R}^{lw}$:
\begin{align}
    X_F=\texttt{flatten}(X)=\begin{pmatrix}
        x_1\\
        \vdots\\
        x_l
    \end{pmatrix}_{lw\times 1}
\end{align}

In addition, we will consider an arbitrary $p\times q$ convolutional filter $C$ made up of rows $c_1^T$, ..., $c_p^T$. With these pieces, we show a numerical example to motivate the path to manipulating a 2D convolution operation into a matrix-vector product.

\begin{example}
    Suppose we have a small gray-scale image $X\in\mathbb{R}^{4\times 4}$ of dimensions a convolution $C\in\mathbb{R}^{3\times 2}$ with the resulting feature map $Z=C*X$:
    \begin{align}
        X&=\begin{pmatrix}
            3 & 4 & 1 & 2\\
            0 & 0 & 5 & 6\\
            2 & 1 & 0 & 3\\
            1 & 4 & 2 & 5
        \end{pmatrix}
        & C=\begin{pmatrix}
            2 & 1 \\
            4 & 3 \\
            -2 & 1            
        \end{pmatrix}
        && Z=C*X=\begin{pmatrix}
            7 & 22 & 45\\
            13 & 3 & 26
        \end{pmatrix}
    \end{align}
Alternatively, the operation can be defined as am matrix-vector product $Z_F=C_MX_F$ resulting in a flattened version of the pre-activation $Z$, where
\begin{align}
    C_MX_F &=
    \left(
        \begin{array}{cccc:cccc:cccc:cccc}
            2 & 1 & 0 & 0 & 4 & 3 & 0 & 0 & -2 & 1 & 0 & 0 & 0 & 0 & 0 & 0\\
            0 & 2 & 1 & 0 & 0 & 4 & 3 & 0 & 0 & -2 & 1 & 0 & 0 & 0 & 0 & 0\\
            0 & 0 & 2 & 1 & 0 & 0 & 4 & 3 & 0 & 0 & -2 & 1 & 0 & 0 & 0 & 0\\ \hdashline
            0 & 0 & 0 & 0 & 2 & 1 & 0 & 0 & 4 & 3 & 0 & 0 & -2 & 1 & 0 & 0\\
            0 & 0 & 0 & 0 & 0 & 2 & 1 & 0 & 0 & 4 & 3 & 0 & 0 & -2 & 1 & 0\\
            0 & 0 & 0 & 0 & 0 & 0 & 2 & 1 & 0 & 0 & 4 & 3 & 0 & 0 & -2 & 1
        \end{array}
    \right)
    \begin{pmatrix}
        3\\
        4\\
        1\\
        2\\ \hdashline
        0\\
        0\\
        5\\
        6\\ \hdashline
        2\\
        1\\
        0\\
        3\\ \hdashline
        1\\
        4\\
        2\\
        5
    \end{pmatrix}
    =\begin{pmatrix}
        7\\
        22\\
        54\\ \hdashline
        13\\
        3\\
        26
    \end{pmatrix} = Z_F
\end{align}
Annotating the blocks, this simplifies as
\begin{align}
        C_MX_F &= \begin{pmatrix}
        B_1 & B_2 & B_3 & 0\\
        0 & B_1 & B_2 & B_3
    \end{pmatrix}
    \begin{pmatrix}
        x_1\\
        x_2\\
        x_3\\
        x_4
    \end{pmatrix}=\begin{pmatrix}
        z_1\\
        z_2
    \end{pmatrix} = Z_F
\end{align}

Reshaping $Z_F$ to $2\times 3$, we reconstruct $Z$ as follows.
\begin{align}
    Z = \texttt{reshape}\left(Z_F\right)=\begin{pmatrix}
        z_1^T\\
        z_2^T
    \end{pmatrix}
    = \begin{pmatrix}
            7 & 22 & 45\\
            13 & 3 & 26
    \end{pmatrix}.
\end{align}

Next, use of \thref{DenseEntropyProp} to measure entropy propagation through a multiplication by a constant matrix $C_M$ requires us to construct a square version of $C_M$. We do this by adding identity matrices on the lower right of each block $B_j$ as per \eqref{SquarifyMatrix} in the case where $N<d$ ($I_1$ in this case). Then, we do the same with $C_M$ consisting of the square $B_j'$ blocks, appending zero blocks and identity matrices ($I_8$ in this case). This results in the following.
\begin{align}
    C_M' &= \begin{pmatrix}
        B_1' & B_2' & B_3' & 0\\
        0 & B_1' & B_2' & B_3'\\
        0 & 0 & I_4 & 0\\
        0 & 0 & 0 & I_4
    \end{pmatrix}\notag
    \\&=\left(
        \begin{array}{cccc:cccc:cccc:cccc}
            2 & 1 & 0 & 0 & 4 & 3 & 0 & 0 & -2 & 1 & 0 & 0 & 0 & 0 & 0 & 0\\
            0 & 2 & 1 & 0 & 0 & 4 & 3 & 0 & 0 & -2 & 1 & 0 & 0 & 0 & 0 & 0\\
            0 & 0 & 2 & 1 & 0 & 0 & 4 & 3 & 0 & 0 & -2 & 1 & 0 & 0 & 0 & 0\\
            0 & 0 & 0 & 1 & 0 & 0 & 0 & 1 & 0 & 0 & 0 & 1 & 0 & 0 & 0 & 0\\
            \hdashline
            0 & 0 & 0 & 0 & 2 & 1 & 0 & 0 & 4 & 3 & 0 & 0 & -2 & 1 & 0 & 0\\
            0 & 0 & 0 & 0 & 0 & 2 & 1 & 0 & 0 & 4 & 3 & 0 & 0 & -2 & 1 & 0\\
            0 & 0 & 0 & 0 & 0 & 0 & 2 & 1 & 0 & 0 & 4 & 3 & 0 & 0 & -2 &   1\\
            0 & 0 & 0 & 0 & 0 & 0 & 0 & 1 & 0 & 0 & 0 & 1 & 0 & 0 & 0 & 1\\
            \hdashline
            0 & 0 & 0 & 0 & 0 & 0 & 0 & 0 & 1 & 0 & 0 & 0 & 0 & 0 & 0 & 0\\
            0 & 0 & 0 & 0 & 0 & 0 & 0 & 0 & 0 & 1 & 0 & 0 & 0 & 0 & 0 & 0\\
            0 & 0 & 0 & 0 & 0 & 0 & 0 & 0 & 0 & 0 & 1 & 0 & 0 & 0 & 0 & 0\\
            0 & 0 & 0 & 0 & 0 & 0 & 0 & 0 & 0 & 0 & 0 & 1 & 0 & 0 & 0 & 0\\
            \hdashline
            0 & 0 & 0 & 0 & 0 & 0 & 0 & 0 & 0 & 0 & 0 & 0 & 1 & 0 & 0 & 0\\
            0 & 0 & 0 & 0 & 0 & 0 & 0 & 0 & 0 & 0 & 0 & 0 & 0 & 1 & 0 & 0\\
            0 & 0 & 0 & 0 & 0 & 0 & 0 & 0 & 0 & 0 & 0 & 0 & 0 & 0 & 1 & 0\\
            0 & 0 & 0 & 0 & 0 & 0 & 0 & 0 & 0 & 0 & 0 & 0 & 0 & 0 & 0 & 1\\
        \end{array}
    \right)
\end{align}
Since $C_M'$ is an upper diagonal square matrix, its determinant is the product of its diagonal elements. In general, it will be equal to $c_{11}$ raised to the power of the number of rows of blocks (2) times the number of rows per block (3), where $c_{11}$ is the upper left term in the convolutional filter $C$. Here, $\det C_M' = c_{11}^{2\cdot 3}=2^6=64$. This determinant allows use of \thref{entropy_matrix_mult} to measure entropy propagation.
\hfill\qed
\end{example}

Generalizing the pattern observed in the example, we can construct the matrix $C_M$ as
\begin{align}
    C_M=\begin{pmatrix}
        B_1 & B_2 & B_3 & \cdots & B_p & 0 & 0 & \cdots & 0\\
        0 & B_1 & B_2 & \cdots & B_{p-1} & R_p & 0 & \cdots & 0\\
        0 & 0 & B_1 & \cdots & B_{p-2} & B_{p-1} & B_p & \cdots & 0\\
        \vdots & \vdots & \vdots & \ddots & \vdots & \vdots & \vdots & \ddots & \vdots\\ 
        0 & 0 & 0 & \cdots & B_1 & B_2 & B_3 & \cdots & B_p
    \end{pmatrix}%_{(l-p+1)(w-q+1)\times lw},
    =\begin{pmatrix}
        C_M^s & C_M^r
    \end{pmatrix},
\end{align}
where
\begin{align}
    B_j=\begin{pmatrix}
        c_{j1} & c_{j2} & c_{j3} & \cdots & c_{jq} & 0 & 0 & \cdots & 0\\
        0 & c_{j1} & c_{j2} & \cdots & c_{j,q-1} & c_{jq} & 0 & \cdots & 0\\
        0 & 0 & c_{j1} & \cdots & c_{j,q-2} & c_{j,q-1} & c_{jq} & \cdots & 0\\
        \vdots & \vdots & \vdots & \ddots & \vdots & \vdots & \vdots & \ddots & \vdots\\
        0 & 0 & 0 & \cdots & c_{j1} & c_{j2} & c_{j3} & \cdots & c_{jq}
    \end{pmatrix}%_{(w-q+1)\times w}
    =\begin{pmatrix}
        B_j^s & B_j^r
    \end{pmatrix}.
\end{align}
The matrix $C_M$ and the blocks $B_j$ both have more columns (width) than rows (length). In both cases, a superscript of $s$ indicates the largest square submatrix made by taking columns 1 through the length of the matrix. For example, the square portion $B_j^s$ is the first $w-q+1$ columns of $B_j$ and $B_j^r$ is the remaining $q-1$ columns (the rectangular portion). In other words, $C_M$ is a $p$-diagonal block Toeplitz matrix with blocks $B_1$, ..., $B_p$, which are $q$-diagonal Toeplitz matrices.

Similar to \thref{DenseEntropyProp}, we convert $C_M$ into a square version whose determinant is determined entirely by the elements of $C_M$. First, the blocks are adjusted to $w\times w$ as
\begin{align}
    B_j'&=\begin{pmatrix}
        B_j^s & B_j^r \\
        0 & I_{q-1}
    \end{pmatrix},%_{w\times w}
\end{align}
Then, we have the resulting $lw\times lw$ matrix $C_M'$,
\begin{align}
    C_M'&=\begin{pmatrix}
        C_{sq}^s & C_{sq}^r\\
        0 & I_{(p-1)w}
    \end{pmatrix}%_{lw\times lw}
\end{align}

The output pre-activation map $Z$ is now obtained via the matrix-vector product $C_M'X_C$ if we then reshape the result into shape $(l-p+1) \times (w-q+1)$ by transposing each successive components into rows. Thus, applying a 2D convolution is equivalent to matrix-vector products, just like dense layers. This gives the following corollary to the previous theorem.

% we need to explain why the terms the I's multiply are not relevant (it comes straight from the Theorem 2)

\begin{corollary}\thlabel{ConvEntropyProp}
    Suppose $X:\Omega\to\mathbb{R}^{l\times w}$ is a random matrix with rows $X_1^T, ..., X_l^T$ and $C\in\mathbb{R}^{p\times q}$ is a fixed convolutional filter, then
    \begin{align}
        H(C*X) &= H(X) + \log\left(\left|\det C_M'\right|\right)\notag
        \\&= H(X) + (w-q+1)(l-p+1)\log\left(\left|c_{11}\right|\right),
    \end{align}
    where $c_{11}$ is the upper left element of the convolutional filter.
\end{corollary}

Hence, the change in entropy in applying a single 2D convolutional filter is proportional to the difference between the filter and image width, difference between filter and image length, and the logarithm of the magnitude of the first weight in the filter.
%As such, entropy propagates from input to pre-activation in a convolutional layer as per this formula. 

\section{Entropy-Based Guidance of Dense and Convolutional Neural Networks}

The prior section established formulas for the entropy of hidden representations within dense and convolutional neural networks. In each case, the entropy is unknown, but the change in entropy can be computed using known parameters and hyperparameters.

The article proposes to guide the training of neural networks to produce ideal entropy propagation patterns. This provides a new lens through which models can be constructed, trained, and tuned. In addition, we provide two distinct tools to guide the training by controlling entropy propagation through:
\begin{enumerate}
    \item Dense layers via a loss using determinants of modified weight matrices $W'$.
    \item 2D convolutional layers via a loss using determinants of modified convolutional operations using $C_M'$.
\end{enumerate}

These new loss terms will be used to construct a compound loss function:
\begin{align}
    L(\mathcal{C},\mathcal{W}) = L_{\text{acc}}(\mathcal{C},\mathcal{W}) + \lambda_1L_{\text{dense}}(\mathcal{W}) + \lambda_2L_{\text{conv}}(\mathcal{C}),
\end{align}\label{compound_loss}
where $\mathcal{C}$ consists of the 2D convolutional filters, and $\mathcal{W}$ consists of the weight matrices of the dense layers.

$L_\text{acc}$ is a standard loss measuring error of the model's primary task; for example, mean squared error (MSE) for regression or cross-entropy for classification. $L_\text{dense}$, and $L_\text{conv}$ are the entropy-based loss terms from dense and convolutional layers, respectively. The hyperparameters $\lambda_1$ and $\lambda_2$ control the strengths of the loss terms relative to the primary loss $L_\text{acc}$. The specific formulas and variations for each entropy-based loss term is established in the following three subsections. Note the losses below generalize this to accommodate fine-grained tuning with layer-specific and channel-specific hyperparameters $\lambda_1^\ell$ (for dense layer $\ell$) and $\lambda_2^{\ell d}$ (for filter $d$ in convolutional layer $\ell$).

\subsection{Dense Entropy Loss}

From \thref{DenseEntropyProp}, we determined how entropy changes as dense layers process data. It is established that the term
\begin{align}
    \log\left(\left|\det W\right|\right) 
\end{align}
describes the change in entropy as a dense layer with weight matrix with square part $W$. Suppose a dense network has layers with weight matrices with square parts $W_1$, $W_2$, ..., $W_L$. Then, an entropy-based loss term is:
\begin{align}
    L_\text{dense}(\mathcal{W})=-\sum\limits_\ell \lambda_1^{\ell}\log\left(\left|\det W_\ell\right|\right)
\end{align}
where $\lambda_1^{\ell}\in\mathbb{R}$ are hyperparameters controlling the strength of the penalty in the ${\ell}^\text{th}$ dense layer, generalizing the single $\lambda_1$ hyperparameter if losses are applied layer-wise. When $\lambda_1^\ell>0$, the added loss term encourages weight matrices to reduce entropy. This provides the opportunity to minimize loss further or faster, and opens the door for smaller, more efficient architectures.
    
\subsection{2D Convolutional Entropy Loss}

According to \thref{ConvEntropyProp}, the change in entropy as an input is processed by a 2D convolution is
\begin{align}
    (l-p+1)(w-q+1)\log |c_{11}|,
\end{align}
so we introduce a loss term proportional to $\log |c_{11}|$ for each convolutional filter.

Suppose the term $c_{ij}^{\ell d}$ corresponds to the element in position $(i,j)$ of the $d$th convolutional filter of the $\ell$th convolutional layer. Then, the entropy-based loss term is
\begin{align}
    L_\text{conv}(\mathcal{C})=-\sum\limits_{\ell,d} \lambda_2^{\ell d}\log\left(\left|c_{11}^{\ell d}\right|\right)
\end{align}
where $\lambda_2^{\ell d}\in\mathbb{R}$ are layer- and channel-wise weighting hyperparameters, generalizing the simpler global $\lambda_2$ from \eqref{compound_loss}. If the $\lambda_2^{\ell d}>0$ decreases in entropy will be penalized, resulting in entropy amplification/preservation. Similarly, $\lambda_2^{\ell d}<0$ results in entropy suppression.

Note that we are considering all convolutions as convolving over single-channel 2D inputs. However, 2D convolutional filters running over multi-channel inputs (e.g. RGB images) are equivalent to several parallel 2D convolutions, or convolutions over different channels. Hence, the formula above works regardless of the number of channels of the input.

\section{Experiments}

We first carried out some qualitative analysis of entropy propagation patterns in well-trained, effective neural networks.

\begin{figure}[H]
    \centering
    \includegraphics[width=\textwidth]{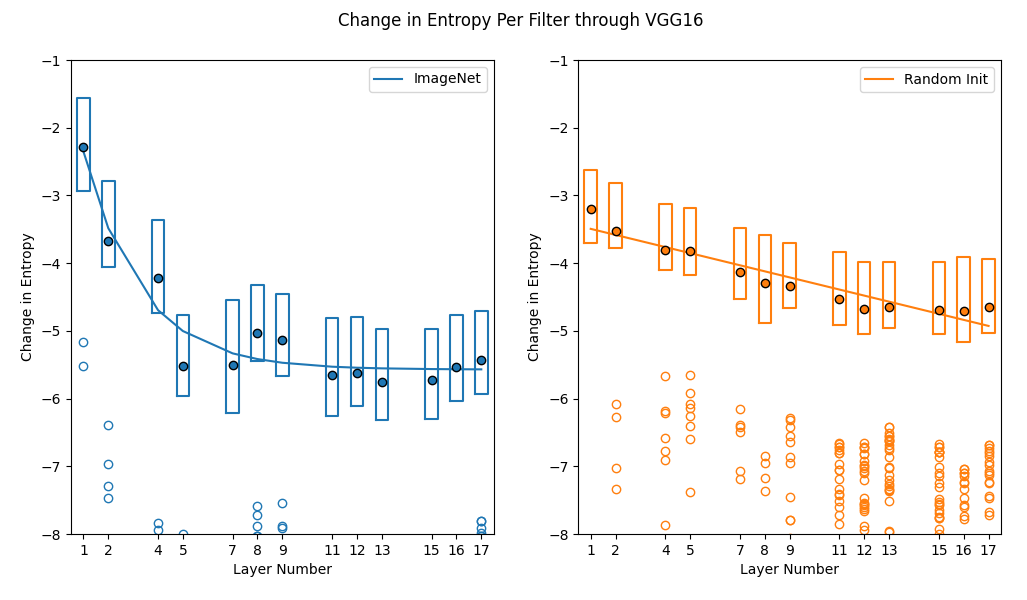}
    \caption{Average change in entropy per filter in each layer of two VGG16 networks, one trained to classify ImageNet and one randomly initialized. The closed dots are means, box plots show first and third quartiles of entropy change per filter at each layer, and outliers are plotted as open dots.}\label{VGG_entropy_gain}
\end{figure}

The left panel of Figure~\ref{VGG_entropy_gain} displays the change in entropy per filter at each layer within a VGG16 CNN well-trained to classify the ImageNet dataset from \citet{deng_imagenet_2009}. It preserves most of its entropy in its early convolutional layers, with mean entropy drops of only around -2.2 in the first layer. The entropy then drops precipitously as data propagates through the later convolutional layers. The entropy drops magnitudes drop exponentially and settle near -5.5 per filter.  In contrast, the right panel shows that the randomly initialized (untrained) network fails to preserve early entropy as strongly and experiences entropy drops that increase only linearly in magnitude in later layers.

There is also an interesting pattern in the outlier filters. The ImageNet-trained network has few outliers, all of which correspond to larger drops in entropy than the norm. They are rare especially later in the network. In contrast, the randomly initialized network has far more outliers, again all corresponding to reductions in entropy. Further the outliers are especially common in the later layers. These trends suggest well-trained networks learn filters reduce entropy more uniformly across filters.

These pattern was observed across multiple well-trained neural networks, and hence we hypothesize penalizing entropy decay in early layers and encouraging entropy drops in later layers would promote better performance.

\subsection{Experimental Modification of Entropy-based Losses}

The dense entropy-based loss term includes $-\log(|\det W|)$, which approaches infinity exponentially if $|\det W|$ approaches 0. To avoid exploding gradients, weights tend to be small, resulting in frequent tiny determinants. Even worse, the convolutional entropy-based loss term is a sum of $-\log\left(|c_{11}^{\ell d}|\right)$ terms, which approach infinity if even a single $c_{11}^{\ell d}$ approaches 0.

To sidestep this issue, we frequently substitute the entropy-based loss terms above with similar functions that are more stable as follows. For the dense case:
\begin{align}
    \frac{1}{|\det W|+\varepsilon}
\end{align}
and for the convolutional case:
\begin{align}
    \frac{1}{|c_{11}^{\ell d}|+\varepsilon}
\end{align}

Figure~\ref{modified_loss} shows the curves behave similarly near 0, but the loss terms modified with $\varepsilon$ do not explode, even with the tiny determinants and $|c_{11}|$ values seen during training. \footnote{Note $\varepsilon=0.5$ is too large in practice and we use $\varepsilon<10^{-3}$, but it was chosen to make the plot more easily visible.} Though the curves have different signs for inputs outside $[-1,1]$, these values are quite uncommon to ensure the stability of backpropagation.

\begin{figure}[H]
    \centering
    \includegraphics[width=0.9\textwidth]{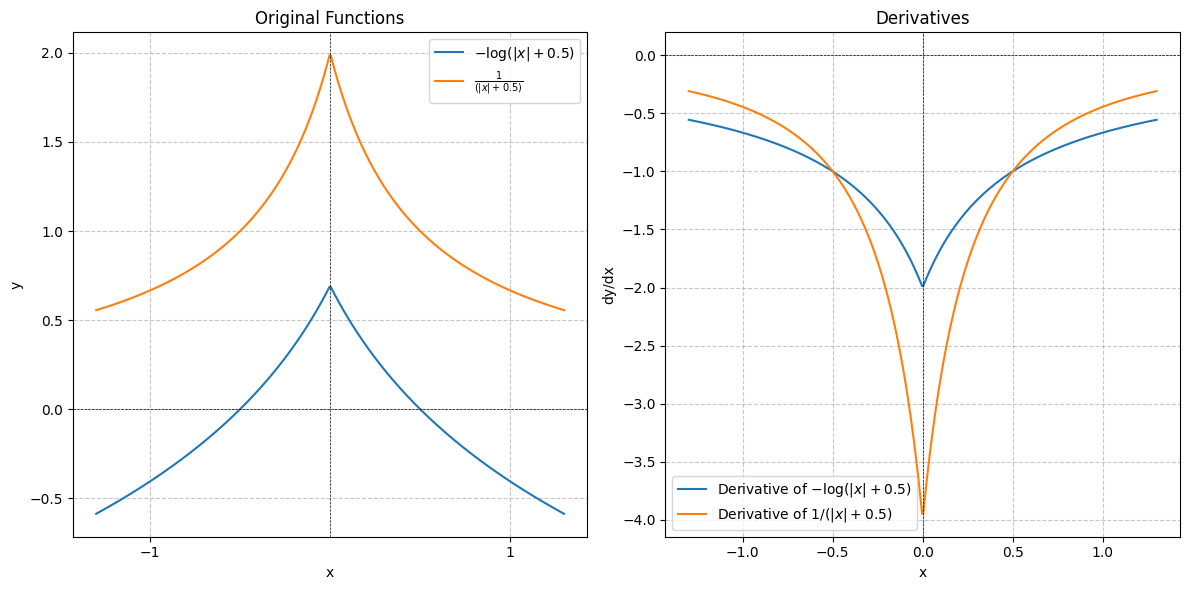}
    \caption{This plot shows the comparison of $-\log(|x|)$ versus $\frac{1}{|x|}$.}
    \label{modified_loss}
\end{figure}

In addition, $\frac{1}{|x|}$ has a larger derivative near 0 than $\log(|x|)$. Note in the first panel of Figure~\ref{modified_loss}, $\frac{1}{|x|=\varepsilon}$ approaches its maximum at $x=0$ more steeply. The second panel shows its absolute derivative more extreme in this location. This amplifies the gradients computed during training to help avoid the vanishing gradient problem, encouraging quicker convergence to better minima.

\subsection{Dense Autoencoders for Image Compression}

To test the efficacy of the entropy-based loss function for dense layers, we trained simple autoencoders with different values of $\lambda_1$ and hidden dimensions for image compression using two benchmark datasets: MNIST from \citet{lecun2010mnist} and CIFAR-10 from \citet{Krizhevsky2009LearningML}. The validation MSE $L_\text{acc}$ (i.e., the reconstruction error) and number of iterations required for convergence were then compared.

The autoencoder includes an input layer, one dense hidden layer, and an output layer. It maps the input data into a latent space, passes this latent representation through an activation function, and then reconstructs the input data with another dense layer. The overall goal is to reduce dimensionality of the input data by maximizing the total variation of that latent space for effective reconstruction. Models such as these can be beneficial in decreasing model sizes, detecting anomalies, denoising, and making downstream tasks more interpretable.

Autoencoder models were trained on each dataset separately with the same activation and optimizer to ensure fair comparison. In this experiment, we compare results on latent dimension widths of 20, 60, 100, ..., 260 and $\lambda_1$ values in $\{0, 0.0001, 0.001, 0.01, 0.1, 1, 10\}$, Each autoencoder uses Adam optimizer, sigmoid activation, and MSE for base loss $L_\text{acc}$. \hlbreakable{We use early stopping to end training when the validation loss saturates for 7 epochs. Each model is trained 10 times with different random initializations to enable statistical comparisons.}

Use of the the dense entropy-based loss term results in convergence up to 4 times faster to minima within $10^{-2}$ as measured by validation reconstruction error (MSE), often with a lower MSE than models trained without entropy loss. \hlbreakable{A more detailed breakdown of the results are presented in Figures~\ref{fig:MNIST_AE_MSE}-\ref{fig:CIFAR_AE_SE} and the corresponding discussion below.}

For the first set of experiments, we use the MNIST benchmark dataset that includes 70,000 labeled images of handwritten digits from 0 to 9 with train/test split of 60,000/10,000 images. Each image is grayscale, contains a centered handwritten digit, and is 28x28 pixels (784 dimensions).

% \begin{figure}[H]
%     \centering
%     \includegraphics[width=0.9\textwidth]{DenseLossPlots.png}
%     \caption{This plot shows how the addition of $L_\text{dense}(\mathcal{W})$  affects the training speed and accuracy of an autoencoder on MNIST and CIFAR-10 data. }
%     \label{dense_loss_plots}
% \end{figure}

\begin{figure}[H]
    \centering
    \begin{subfigure}[c]{0.55\textwidth}
        \includegraphics[width=\textwidth]{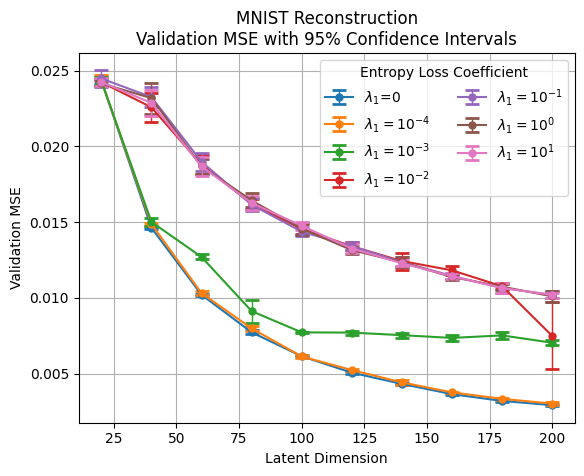}
        \caption{\hlbreakable{Validation MSE by Latent Dimension and $\lambda_1$}}
    \end{subfigure}%
    \hfill
    \begin{subfigure}[c]{0.42\textwidth}
        \caption*{\hlbreakable{$\quad\quad\quad\quad$Validaton MSE $t$-tests for different $\lambda_1$}}
        \begin{tabular}{c|c|*{7}{p{0.4cm}}|}
            \multicolumn{1}{c}{} & \multicolumn{8}{c}{$\lambda_{12}$} \\ % Title for the entire table spanning all columns
            \cline{2-9}
            \multirow{10}{*}{\rotatebox{90}{$\lambda_{11}$}} & & $0$ & $10^{-4}$ & $10^{-3}$ & $10^{-2}$ & $10^{-1}$ & $10^{0}$ & $10^{1}$ \\
            \cline{2-9}
            & $0$ & & $-$ & $-$ & $-$ & $-$ & $-$ & $-$\\
            & $10^{-4}$ & $+$ & & $-$ & $-$ & $-$ & $-$ & $-$\\
            & $10^{-3}$ & $+$ & $+$ & & $-$ & $-$ & $-$ & $-$ \\
            & $10^{-2}$ & $+$ & $+$ & $+$ & & & &\\
            & $10^{-1}$ & $+$ & $+$ & $+$ & & & &\\
            & $10^{0}$  & $+$ & $+$ & $+$ & & & &\\
            & $10^1$ & $+$ & $+$ & $+$ & & & &\\
            \cline{2-9}
        \end{tabular}
        \caption{\hlbreakable{Each row has $t$-test results for the mean MSE for $\lambda_{11}$ compared to each $\lambda_{12}$ across 10 training replications for latent dimension 80.}\vspace{3cm}
        }
        \label{tab:table_mnist_ae_mse}
    \end{subfigure}
    \caption{\hlbreakable{MNIST Reconstruction MSE for Different Latent Dimensions and Dense Entropy Loss Coefficients $\lambda_1$}}
    \label{fig:MNIST_AE_MSE}
\end{figure}

\hlbreakable{First, Figure~\ref{fig:MNIST_AE_MSE} provides an analysis of the effects on validation MSE upon modifying the size of the latent dimension and hyperparameter $\lambda_1$ for the entropy loss on the first weight matrix of the autoencoder. Each curve in Figure~\ref{fig:MNIST_AE_MSE}(a) shows the mean validation MSE at convergence over 10 replications for each choice of latent dimension with 95\% $t$-confidence intervals. The individual plots correspond to different entropy loss weights $\lambda_1$. With a couple of exceptions, the converged validation MSE bifurcates into two groups as latent dimension increases. Lower values of $\lambda_1$ result in MSE around 0.003 while larger $\lambda_1$'s tend to converge to MSE near 0.010, both exceptionally small values indicating high-quality reconstruction. In addition, despite randomness in the training (i.e., the weight initialization and sampling by the Adam optimizer), the relatively small confidence intervals in almost all cases indicate convergence to consistent validation MSE scores across replications.}

\hlbreakable{Figure~\ref{fig:MNIST_AE_MSE}(b) shows a table of the results of $t$-tests for comparing the validation MSE for each pair of samples of training replications with $\lambda_1=\lambda_{11}$ and $\lambda_1=\lambda_{12}$. More explicitly, we perform the $t$-tests for the following null and alternative hypotheses:
\begin{align*}
    &H_0: \mu_{\text{Val MSE for }\lambda_{11}}=\mu_{\text{Val MSE for }\lambda_{12}}\\
    &H_1: \mu_{\text{Val MSE for }\lambda_{11}}\neq\mu_{\text{Val MSE for }\lambda_{12}}
\end{align*}
Each cell in the table has one of three indicators:
\begin{itemize}
    \item Empty cell: $p>0.01$, there is not a statistically significant difference in mean validation MSE at a $\alpha=0.01$ significance threshold
    \item Plus sign ($+$): $p<0.01$ indicating a statistically significant difference in mean validation MSE, and it is significantly \textit{higher} for $\lambda_{11}$ than for $\lambda_{12}$
    \item Minus sign ($-$): $p<0.01$ indicating a statistically significant difference in mean validation MSE, and it is significantly \textit{lower} for $\lambda_{11}$ than for $\lambda_{12}$
\end{itemize}
The table contains details at only latent dimension 80, but similar patterns are seen in almost all other latent dimensions. Broadly, we see lower values of $\lambda_1$ (0, $10^{-4}$, $10^{-3}$) lead to statistically significantly lower validation MSE, although the difference in MSE is on the order of $0.001$ across all cases.}

\hlbreakable{Below, Figure~\ref{fig:MNIST_AE_SE} compares the stopping epochs for each pair of latent dimension and $\lambda_1$ value. The stopping epoch is defined as the epoch when early stopping ends the training run due to saturated validation MSE for 7 epochs. The runtimes of each train epoch was nearly the same across $\lambda_1$ values for fixed latent dimension. Hence, stopping epoch is directly proportional to training time.}

\hlbreakable{In the graph Figure~\ref{fig:MNIST_AE_SE}(a), we see the stopping epochs are much lower for $\lambda_1\geq 10^{-4}$, especially for latent dimensions over 100. Mean stopping epoch is over 40 for the larger models without entropy loss ($\lambda_1=0$) and only 8-10 for $\lambda_1\geq 10^{-4}$, a reduction of 80\% training epochs, and hence training time. The statistical significance of these differences are reported in Figure~\ref{fig:MNIST_AE_SE}--an exact inverse of the significance patterns for MSE. We conclude that, while larger $\lambda_1$ values increase validation MSE marginally, they enable huge reductions in training time.}

\begin{figure}[H]
    \centering
    \begin{subfigure}[c]{0.55\textwidth}
        \includegraphics[width=\textwidth]{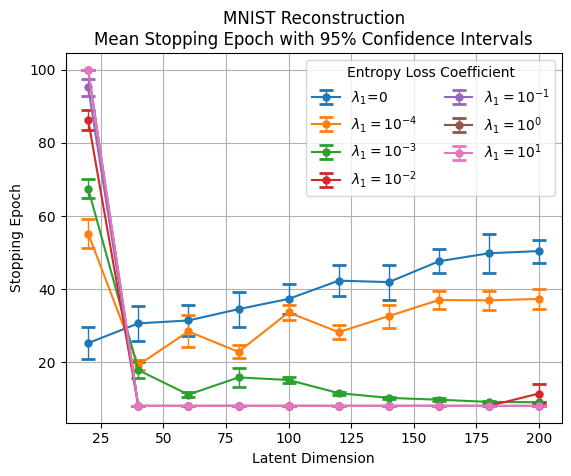}
        \caption{Stopping Epoch by Latent Dimension and $\lambda_1$}
    \end{subfigure}%
    \hfill
    \begin{subfigure}[c]{0.42\textwidth}
        \caption*{$\quad\quad\quad\quad$Stopping Epoch $t$-tests for different $\lambda_1$}
        \begin{tabular}{c|c|*{7}{p{0.4cm}}|}
            \multicolumn{1}{c}{} & \multicolumn{8}{c}{$\lambda_{12}$} \\ % Title for the entire table spanning all columns
            \cline{2-9}
            \multirow{10}{*}{\rotatebox{90}{$\lambda_{11}$}} & & $0$ & $10^{-4}$ & $10^{-3}$ & $10^{-2}$ & $10^{-1}$ & $10^{0}$ & $10^{1}$ \\
            \cline{2-9}
            & $0$ & & $+$ & $+$ & $+$ & $+$ & $+$ & $+$\\
            & $10^{-4}$ & $-$ & & $+$ & $+$ & $+$ & $+$ & $+$\\
            & $10^{-3}$ & $-$ & $-$ & & $+$ & $+$ & $+$ & $+$\\
            & $10^{-2}$ & $-$ & $-$ & $-$ & & & & \\
            & $10^{-1}$ & $-$ & $-$ & $-$ & & & & \\
            & $10^{0}$  & $-$ & $-$ & $-$ & & & & \\
            & $10^{1}$  & $-$ & $-$ & $-$ & & & & \\
            \cline{2-9}
        \end{tabular}
                \caption{Each row has $t$-test results for the mean stopping epoch for $\lambda_{11}$ compared to each $\lambda_{12}$ across 10 training replications for latent dimension 180.\vspace{3cm}
        }
        \label{tab:table_mnist_ae_se}
    \end{subfigure}
    \caption{MNIST Stopping Epoch for Different Latent Dimensions and Entropy Loss Coefficients $\lambda_1$}
    \label{fig:MNIST_AE_SE}
\end{figure}

\hlbreakable{Next, Figures~\ref{fig:CIFAR_AE_MSE}-\ref{fig:CIFAR_AE_SE} provide a similar analysis on the CIFAR-10 dataset}, a benchmark dataset of 60,000 tiny color images belonging to 10 classes of objects, including cats and airplanes. Images are all RGB and 32x32 pixel format with a total dimension of 3072. Training on this higher dimensional dataset allows us to assess the implications of the input dimensionality compared to MNIST. 

\begin{figure}[H]
    \centering
    \begin{subfigure}[c]{0.55\textwidth}
        \includegraphics[width=\textwidth]{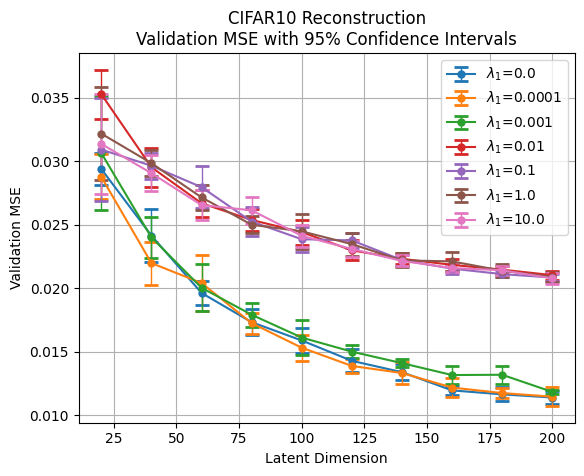}
        \caption{Validation MSE by Latent Dimension and $\lambda_1$}
    \end{subfigure}%
    \hfill
    \begin{subfigure}[c]{0.42\textwidth}
        \caption*{$\quad\quad\quad\quad$Validaton MSE $t$-tests for different $\lambda_1$}
        \begin{tabular}{c|c|*{7}{p{0.4cm}}|}
            \multicolumn{1}{c}{} & \multicolumn{8}{c}{$\lambda_{12}$} \\ % Title for the entire table spanning all columns
            \cline{2-9}
            \multirow{10}{*}{\rotatebox{90}{$\lambda_{11}$}} & & $0$ & $10^{-4}$ & $10^{-3}$ & $10^{-2}$ & $10^{-1}$ & $10^{0}$ & $10^{1}$ \\
            \cline{2-9}
            & $0$       & & & & $-$ & $-$ & $-$ & $-$\\
            & $10^{-4}$ & & & & $-$ & $-$ & $-$ & $-$ \\
            & $10^{-3}$ & & & & $-$ & $-$ & $-$ & $-$ \\
            & $10^{-2}$ & $+$ & $+$ & $+$ & & & &\\
            & $10^{-1}$ & $+$ & $+$ & $+$ & & & &\\
            & $10^{0}$  & $+$ & $+$ & $+$ & & & &\\
            & $10^1$ & $+$ & $+$ & $+$ & & & &\\
            \cline{2-9}
        \end{tabular}
            \caption{Each row has $t$-test results for the mean validation MSE for $\lambda_{11}$ compared to each $\lambda_{12}$ across 10 training replications for latent dimension 80.\vspace{3cm}
            }
    \label{tab:table_cifar_ae_mse}
    \end{subfigure}
    \caption{CIFAR10 Reconstruction MSE for Different Latent Dimensions and Entropy Loss Coefficients $\lambda_1$}
    \label{fig:CIFAR_AE_MSE}
\end{figure}

\hlbreakable{Just like MNIST, Figure~\ref{fig:CIFAR_AE_MSE} demonstrates validation MSE for CIFAR-10 bifurcates as latent dimension grows into two groups: lower MSE for $\lambda_1\leq 0.001$ and marginally higher MSE for $\lambda_1\geq 0.01$. The two groups are not only visible in the graph Figure~\ref{fig:CIFAR_AE_MSE}(a) but actually perfectly align statistically, as we see in Figure~\ref{fig:CIFAR_AE_MSE}(b) for latent dimension 200. Models with $\lambda_1\leq 0.001$ do not have statistically significant differences in validation MSE with one another, but they are statistically significantly lower than models with $\lambda_1\geq 0.01$. Similarly, the latter group has statistically equivalent validation MSE.}

\begin{figure}[H]
    \centering
    \begin{subfigure}[c]{0.55\textwidth}
        \includegraphics[width=\textwidth]{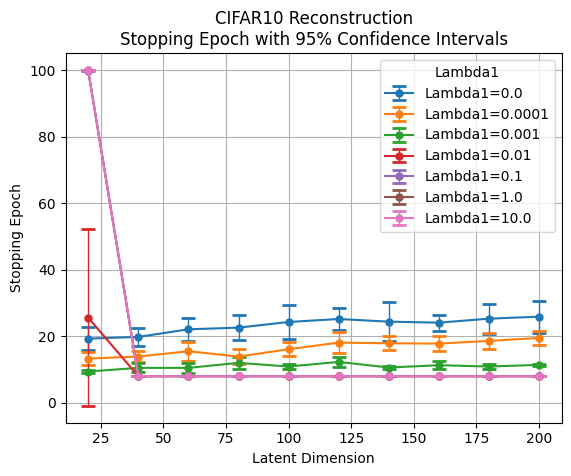}
        \caption{Stopping Epoch by Latent Dimension and $\lambda_1$}
    \end{subfigure}%
    \hfill
    \begin{subfigure}[c]{0.42\textwidth}
        \caption*{$\quad\quad\quad\quad$Stopping Epoch $t$-tests for different $\lambda_1$}
        \begin{tabular}{c|c|*{7}{p{0.4cm}}|}
            \multicolumn{1}{c}{} & \multicolumn{8}{c}{$\lambda_{12}$} \\ % Title for the entire table spanning all columns
            \cline{2-9}
            \multirow{10}{*}{\rotatebox{90}{$\lambda_{11}$}} & & $0$ & $10^{-4}$ & $10^{-3}$ & $10^{-2}$ & $10^{-1}$ & $10^{0}$ & $10^{1}$ \\
            \cline{2-9}
            & $0$ & & $+$ & $+$ & $+$ & $+$ & $+$ & $+$\\
            & $10^{-4}$ & $-$ & & $+$ & $+$ & $+$ & $+$ & $+$\\
            & $10^{-3}$ & $-$ & $-$ & & $+$ & $+$ & $+$ & $+$\\
            & $10^{-2}$ & $-$ & $-$ & $-$ & & & & \\
            & $10^{-1}$ & $-$ & $-$ & $-$ & & & & \\
            & $10^{0}$  & $-$ & $-$ & $-$ & & & & \\
            & $10^{1}$  & $-$ & $-$ & $-$ & & & & \\
            \cline{2-9}
        \end{tabular}
        \caption{Each row has $t$-test results for the mean stopping epoch for $\lambda_{11}$ compared to each $\lambda_{12}$ across 10 training replications for latent dimension 80.\vspace{3cm}
            }
        \label{tab:table_cifar_ae_se}
    \end{subfigure}
    \caption{CIFAR10 Stopping Epoch for Different Latent Dimensions and Entropy Loss Coefficients $\lambda_1$}
    \label{fig:CIFAR_AE_SE}
\end{figure}

\hlbreakable{Lastly, Figure~\ref{fig:CIFAR_AE_SE} demonstrates a similar pattern as MNIST. Larger values $\lambda_1\geq 0.01$ reach early stopping in 8-10 epochs while an absence of entropy loss ($\lambda_1=0$) requires more than twice as many epochs as latent dimension grows.}

In experiments for both MNIST and CIFAR-10, we see \hlbreakable{statistically lower validation MSE scores for smaller coefficients $\lambda_1$ of the dense entropy loss applied to the first layer of the autoencoders, but these differences are small, on the order of 0.01.} However, use of the entropy loss results in remarkable speedups in training when the autoencoder has a latent dimension of 60 or more, the optimizer finds effective minima much more quickly during training.

\subsection{CNNs for Image Classification with Statistical Evaluation of Convolutional Entropy Loss}

Next, we perform image classification experiments on the benchmark dataset CIFAR-10 with the entropy-based loss function for convolutional layers, $L_{\text{conv}}$. Here, we explore the impact of the convolutional entropy-based loss term with different weighting hyperparameters $\lambda_2$ for small CNN classifiers of varying widths and depths, \hlbreakable{and validate the differences in performance statistically}.

CNNs in these experiments have 1-3 successive blocks containing 1 convolutional layer (filter size $3\times 3$, stride $1\times 1$) and 1 max pooling layer (size $2\times 2$, stride $2\times 2$) followed by a softmax classifier. Leaky ReLU activations are used with each convolutional layer. Widths of 32, 64, and 128 convolutional filters in each layer are tested with each depth. The CNNs use cross-entropy for the base classification loss $L_\text{acc}$ and are trained with the Adam optimizer. Each architecture is trained with varying $\lambda_2$ in $\{0, 0.0001, 0.001, 0.01, 0.1, 1, 10\}$, applied in only the first convolutional layer in one set of experiments and applied to all layers in another set. \hlbreakable{Ten training replications are run for models with each fixed set of hyperparameters to generate random samples for statistical analysis.}

The pattern in entropy changes in the high-quality VGG classifier observed in Figure~\ref{VGG_entropy_gain} prompted a hypothesis that encouraging entropy preservation in the early layers will have a positive impact on classification accuracy. The results of experiments with the entropy-based loss applied only to the first convolutional layer are shown in Table~\ref{CNN_classifier_table}.

\hlbreakable{For each model, we provide the mean base train and accuracies without entropy loss as well as the mean change in accuracies with the optimal hyperparameter $\lambda_2$. The mean change in accuracies is given in the last two columns. In each case, we perform one-tailed $t$-tests of the hypotheses
\begin{align*}
    &H_0: \mu_{\text{val accuracy with entropy loss}}\leq\mu_{\text{val accuracy without entropy loss}}\\
    &H_1: \mu_{\text{val accuracy with entropy loss}}>\mu_{\text{val accuracy without entropy loss}}
\end{align*}
for training accuracy in column 4 and for test accuracy in column 5. The training accuracy mostly do not exhibit statistically significant differences at the $\alpha=0.05$ significance level, with one exception. However, all validation accuracies have statistically significant increases with the convolutional entropy loss.}

\begin{table}[H]
    \centering
    \begin{tabular}{c|l|l|l|l}
        Architecture & Mean Base Train Acc. & Mean Base Val. Acc. & Mean $\Delta$ Train Acc. & Mean $\Delta$ Val. Acc. \\ \hline \hline
        $[32]$ & 0.7387 & 0.6402 & +0.0015 & +0.0052** (0.0035) \\
        $[64]$ & 0.7612 & 0.6452 & +0.0031 & +0.0077* (0.0155)\\
        $[128]$ & 0.7736 & 0.6479 & +0.0076 & +0.0074** (0.0016)\\ \hline
        $[32, 32]$ & 0.7444 & 0.6745 & +0.0081 & +0.0076** (0.0043)\\
        $[64, 64]$ & 0.7924 & 0.6954 & +0.0065 & +0.0052** (0.0065)\\
        $[128, 128]$ & 0.8447 & 0.7061 & $-0.0087$ & +0.0063* (0.0210)\\ \hline
        $[32, 32, 32]$ & 0.7441 & 0.6789 & +0.0125* (0.0179) & +0.0114*** (0.0004) \\
        $[64, 64, 64]$ & 0.8152 & 0.7026 & +0.0054 & +0.0064* (0.0130) \\
        $[128, 128, 128]$ & 0.8890 & 0.7098 & +0.0041 & +0.0053* (0.0151) \\
    \end{tabular}
    \vspace{.5cm}
    \caption{Experimental results for CIFAR10 classification with CNNs (10 replications). The mean base train and validation accuracy for $\lambda_2=0$ is displayed alongside the mean best gain in accuracies with $\lambda_2>0$ applied to the first convolutional layer}
    \label{CNN_classifier_table}
\end{table}

Interestingly, applying the entropy-based loss to all layers with the same weighting hyperparameters yields no significant gains, neither for train nor validation accuracy. This supports the hypothesis that encouraging entropy preservation in the early layers trains the CNN to extract higher-quality latent representations of the input data, enabling better downstream classification performance.

\hlbreakable{The same experiments were performed with MNIST, but it is such an easy dataset to classify that all models performed extremely well and converged quickly, so there was no interesting distinction upon using or not using the convolutional entropy loss. This reinforces the conclusion from the autoencoders indicating the need for a sufficiently high dimensionality to reap benefits from the entropy losses.}

\hlbreakable{\subsection{Large-scale CNNs for Classification with Convolutional Entropy Loss}}

\hlbreakable{Though CNN classifiers for tiny CIFAR-10 images experience reliable performance gains with the entropy loss, its impact is unclear for large CNNs trained to classify large, real-world images. To demonstrate the effects of entropy loss at scale, we train several large CNNs, including variations of VGG-16 developed by \citet{simonyan_very_2015} and ResNet-50 developed by \citet{he2016deep} for image classification. We classify the Imagenette dataset by \citet{Howard_Imagenette_2019}. Imagenette is a 10-class subset of the ImageNet Large Scale Visual Recognition Challenge (ILSVRC 2012) dataset from \citet{deng_imagenet_2009}. The color images range in dimensions in the dataset have an average image resolution of 469x387 pixels.}

\hlbreakable{Figure~\ref{VGG_entropy_gain} demonstrates well-trained VGG-16 classifiers preserve entropy in their early convolutional layers. Hence, we hypothesized promoting entropy in the early layers of VGG-16 during training would yield benefits for a newly-trained classifier, as shown in Figure~\ref{fig:block-diagram}. We ran experiments with entropy loss applied to the first convolutional layer, entropy loss applied to the first two convolutional layers, and a baseline without entropy loss for comparison.}

\begin{figure}
    \centering    \includegraphics[width=\textwidth]{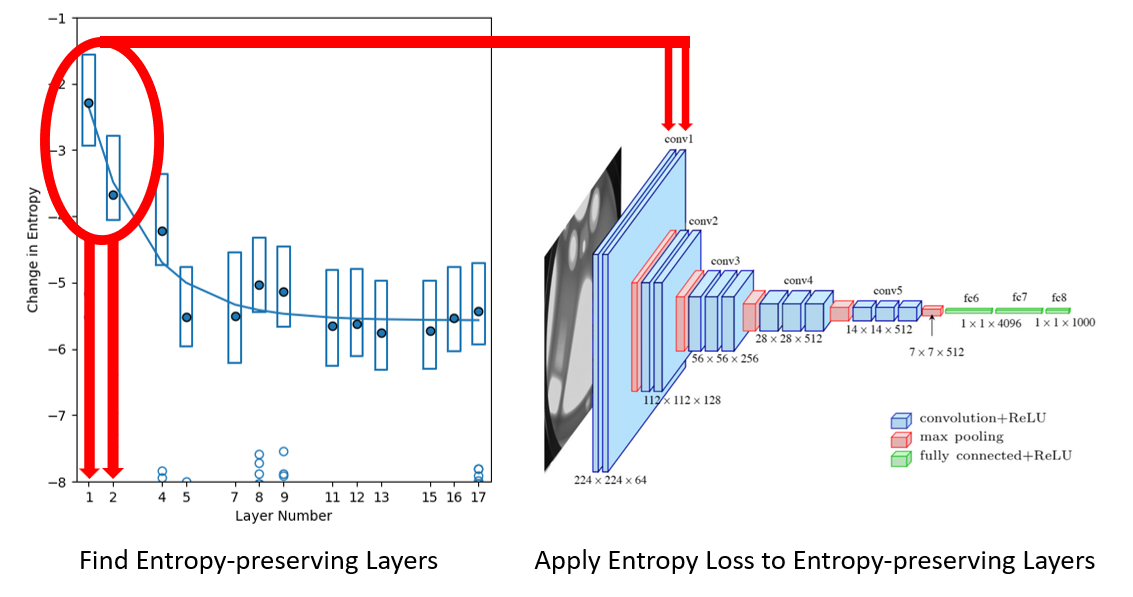}
    \caption{Strategy for architectures with known well-trained weights for a different task: Find the entropy-preserving layers and apply the entropy loss to encourage the same pattern when training for the desired dataset.}\label{fig:block-diagram}
\end{figure}

\hlbreakable{For ResNet, we again train baseline models without entropy loss, models with entropy loss applied to the first convolutional layer (layer 1), and models with entropy loss applied to the first two convolutional layer (layers 1 and 5). Additionally, ResNet has additive residual connections (or skip-connections) that add latent representations from prior layers to later layers. We hypothesize promoting entropy in hidden representations pushed forward by skip connections will preserve information in the added representation. Hence, we have a fourth test case where we apply the entropy loss in these locations, which we call a ``pre-skip'' entropy loss configuration. A fifth case is prompted by Figure~\ref{fig:resnet-entropy-gain}, where we promote entropy on the high-entropy layers below. We call this case the ``empirical`` entropy loss configuration.}

\begin{figure}
    \centering
    \includegraphics[width=\textwidth]{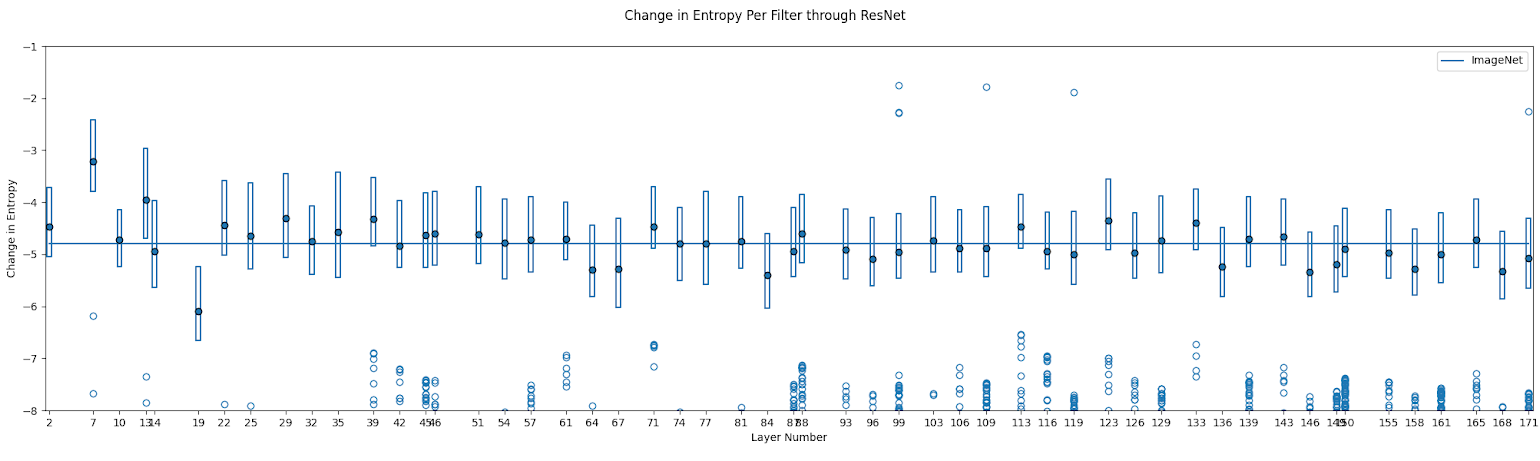}
    \caption{Average change in entropy per filter in each layer of a ResNet-50 network trained on the ImageNet dataset.}
    \label{fig:resnet-entropy-gain}
\end{figure}

\hlbreakable{For each fixed classifier network (VGG-16 or ResNet), training strategy (vanilla, augmentation, augmentation + dropout), and placement of convolutional entropy loss, we trained models with coefficients $\lambda_2\in\{0.0001, 0.001, 0.01, 0.1, 1, 10\}$. Stochastic gradient descent is used for all models with a learning rate of 0.01 which is divided by 10 upon 10 epochs with no reduction in validation accuracy. Early stopping is used when the validation accuracy experiences no improvements in 25 epochs. Augmentations include horizontal flips, up to 20\% random rotation, and up to 20\% random zoom. Dropout rate 0.5 is used between the dense layers in the heads of the models.}

\hlbreakable{The results are displayed in Table~\ref{tab:big_CNN_experiments} below. The validation accuracy is documented for each model with each fixed set of hyperparameters. The maximum validation accuracy for each model is bolded. In 17 of 18 experiments, the maximum validation accuracy occurs when the entropy loss is in effect, and mostly when $\lambda_2\geq 0.001$.}

\begin{table}
  \centering
  \caption{Validation Accuracy for ImageNette Classification by Model and Convolutional Entropy Loss Settings}
  \label{tab:big_CNN_experiments}
  \begin{tabular}{l|cc|p{1.7cm}|ccccccc}
    \toprule
    & & & & \multicolumn{7}{c}{Convolutional Entropy Loss Weight $\lambda_2$} \\
    \cmidrule{5-11}
     \multirow{2}{*}[4pt]{Model} & \multirow{2}{*}[4pt]{A} & \multirow{2}{*}[4pt]{D} & \multirow{2}{*}[10pt]{\shortstack[c]{Entropy \\ Loss Layers}} & $0$ & $0.0001$ & $0.001$ & $0.01$ & $0.1$ & $1$ & $10$ \\
    \midrule
    \multirow{6}{*}{VGG} & & & [1] & 0.698 & 0.752 & \textbf{0.756} & 0.734 & 0.716 & 0.738 & 0.736 \\
    & & & [1,2] & 0.726 & 0.720  & 0.720  & 0.726 & 0.718 & 0.730  & \textbf{0.750}  \\ \cmidrule{2-11}
    & $\checkmark$ & & [1] & 0.804 & 0.796 & \textbf{0.828} & 0.794 & 0.800   & 0.814 & 0.804 \\
    & $\checkmark$ & & [1,2] & 0.816 & 0.816 & 0.806 & \textbf{0.824} & 0.796 & 0.780  & 0.800   \\ \cmidrule{2-11}
    & $\checkmark$ & $\checkmark$ & [1] & 0.854 & 0.854 & \textbf{0.874} & 0.826 & 0.836 & 0.868 & 0.866 \\
    & $\checkmark$ & $\checkmark$ & [1,2] & 0.858 & 0.860  & 0.862 & 0.854 & \textbf{0.864} & 0.854 & 0.848 \\
    \midrule
    \multirow{12}{*}{ResNet} & & & [1] & \textbf{0.792} & 0.766 & 0.782 & 0.764 & 0.784 & 0.784 & 0.786 \\
    & & & [1,5] & 0.776 & 0.782 & \textbf{0.802} & 0.778 & 0.776 & 0.774 & 0.766 \\
    & & & Pre-skip & 0.748 & 0.766 & 0.812 & 0.770 & 0.786 & \textbf{0.814} & 0.778 \\
    & & & Empirical & 0.778 & 0.782 & 0.774 & 0.776 & 0.792 & \textbf{0.810} & 0.776 \\ \cmidrule{2-11}
    & $\checkmark$ & & [1] & 0.786 & 0.784 & 0.778 & 0.752 & \textbf{0.814} & 0.772 & 0.770  \\
    & $\checkmark$ & & [1,5] & 0.780  & 0.782 & \textbf{0.786} & 0.762 & 0.768 & 0.760  & 0.750  \\
    & $\checkmark$ & & Pre-skip & 0.780  & 0.794 & 0.760  & 0.774 & 0.766 & \textbf{0.808} & 0.792 \\
    & $\checkmark$ & & Empirical & 0.784 & \textbf{0.802} & 0.746 & 0.766 & 0.768 & 0.768 & 0.754 \\ \cmidrule{2-11}
    & $\checkmark$ & $\checkmark$ & [1] & 0.776 & 0.802 & 0.796 & 0.798 & \textbf{0.810}  & 0.792 & 0.796 \\
    & $\checkmark$ & $\checkmark$ & [1,5] & 0.776 & \textbf{0.812} & 0.798 & 0.806 & 0.776 & 0.810  & 0.796 \\
    & $\checkmark$ & $\checkmark$ & Pre-skip & 0.792 & 0.802 & 0.796 & 0.794 & \textbf{0.808} & 0.790  & 0.820  \\
    & $\checkmark$ & $\checkmark$ & Empirical & 0.804 & 0.806 & 0.806 & \textbf{0.824} & 0.804 &  0.800 & 0.798 \\
    \bottomrule
  \end{tabular}
\end{table}

\hlbreakable{In VGG, augmentation and dropout have positive impacts to validation accuracy. Keeping these constant, note applying the loss to layer [1] is better than layers [1,2] in all three model structures. This aligns with the first layer preserving the most entropy for VGG16 trained on ImageNet in Figure~\ref{VGG_entropy_gain}. In ResNet, there is no clear best choice for layers to apply the entropy loss, but the loss helps in all cases except the first, and the higheset overall accuracy for ResNet comes with the Empirical layers inspired by the pre-trained ResNet shown in Figure~\ref{fig:resnet-entropy-gain}.}

\hlbreakable{These experiments demonstrate the entropy loss improves the performance of large-scale CNN classifiers for images with realistic resolutions. In addition, the strategy of choosing convolutional layers that preserve entropy in classifiers with the same architecture well-trained for different datasets, as depicted in Figure~\ref{fig:block-diagram}.}

\hlbreakable{\subsection{U-Net for Image Segmentation}}

\hlbreakable{In previous experiments, the entropy change term was visualized in a well-trained network of the same architecture the entropy loss was to be implemented. Regardless of data, as long as the model was well trained for a specific task, the entropy patterns stayed the same for that specific architecture. In most cases, it was shown that in early layers of the networks, the entropy change was high which lead us to implement the entropy loss in those locations to ensure adequate entropy preservation. In this section, we use a significantly different, less straightforward architecture for another task: image segmentation.}

\hlbreakable{Image segmentation is a task that attempts both classification and localization of objects in an image. State of the art CNN architecture, U-Net by \citet{ronneberger2015Unet}, is widely used for performing segmentation. U-Net contains a downsampling path, an upsampling path, with skip connections between these two paths. The downsampling path extracts features while minimizing the size of the image as it moves forward. As it moves into the upsampling path, the data is reconstructed through upconvolutions and features discovered in the downsampling path through the skip connections. U-Net exploits this form of reconstruction which reduces the necessary amount of training data. However, since this model is quite small, it may also lose important information as the image moves through the architecture.}

\hlbreakable{Similary to previous experiments, the entropy change term was calculated through a well trained U-Net. In Figure~\ref{UNet_entropy_gain}, we see that the entropy change is also in a very distinct "U" shape. To follow the original ideal of exploiting the areas known to have high entropy, the entropy loss was implemented in the regions of large entropy change.}

\begin{figure}
    \centering
    \includegraphics[width=\textwidth]{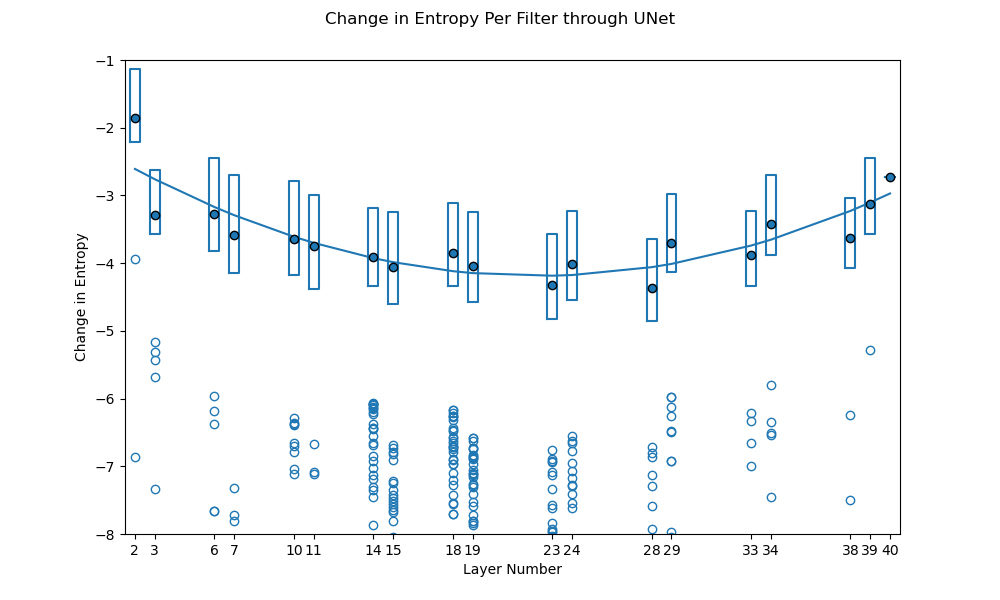}
    \caption{Average change in entropy per filter in each convolutional layer of a U-Net trained for semantic segmentation of the COCO dataset. The closed dots are means, box plots show first and third quartiles of entropy change per filter at each layer, and outliers are plotted as open dots.}
    \label{UNet_entropy_gain}
\end{figure}

\hlbreakable{A U-Net based on the original architecture from \citet{ronneberger2015Unet} was trained on COCO segmentation data, consisting of 60000/10000 training/validation 256x256 images, from \citet{lin_microsoft_2014}. The implementation from \citet{chopra_coco_unet} was adapted to utilize these larger images and the modified loss function. Models were trained at a variety of different layer locations and values of lambda to determine the most impactful placement of the entropy penalty for both accuracy and convergence time. Each model was trained utilizing early stopping with a minimum loss change of 0.01 for 10 training epochs. All hyperparameters outside of the compared $\lambda$'s, were kept the same and a seed was set for the initialization for an equal comparison. }

\hlbreakable{Initial U-Net experiments consisted of placing the entropy loss in multiple locations. Models were trained with the entropy loss applied at each of the following sets of layers:  [2], [2, 3], [3, 7, 11, 15], and [3, 7, 11, 15, 19, 24, 29, 34, 29]. These models often took approximately$\ 30\%$ longer than the baseline model. Additionally, accuracy was often not affected positively. For the case of implementing a entropy loss at layer 2, it likely caused a form of bottleneck, causing the model to take a long time to converge and adding limited valuable information later on.}

\hlbreakable{We saw the most impact in the layers with the highest amount of entropy change from the well trained network. Models with the entropy loss at layers $\{19,24,29,34,39\}$ were trained with varying $\lambda_2\in\{0, 0.0001, 0.001, 0.01, 0.1, 1,10\}$. These layers are feeding in to the up path concatenations which are likely feeding in higher quality information for the reconstruction and prediction. }

\begin{table}[H]
\centering % Center the table
\begin{tabular}{ccccccc} % 7 columns, all centered; change to l, r for left, right align
\toprule
Loss & $\lambda_2$ & $\#$ of Epochs & Training Time (s) & Accuracy & IOU\\ % Table header
\midrule
MSE & 0.00 & 20 & 9276.82 & 0.8290 & 0.6693 \\
MSE + Entropy & 0.0001 & 24 & 13202.78 & \textbf{0.8336} & 0.6823  \\ % Row 3
MSE + Entropy & 0.001 & 21 & \textbf{8304.54} & \textbf{0.8316} & 0.6724\\
MSE + Entropy & 0.01 & \textbf{19} & \textbf{8590.74} & \textbf{0.8316} & 0.6602 \\
MSE + Entropy & 0.1 & 23 & \textbf{9038.20} & \textbf{0.8305} & 0.6626 \\
MSE + Entropy & 1.00 & 23 & 11374.05 & \textbf{0.8349} & 0.6734 \\
MSE + Entropy & 10.00 & 36 & 14358.95 & \textbf{0.8300} & 0.6786 \\
\bottomrule
\end{tabular}
\vspace{.5cm}
\caption{U-Net Experiments with Entropy Loss} % Table caption
\label{table:unet} % Label for referencing
\end{table}

\hlbreakable{A commonly used metric in determining quality of a segmentation model is intersection over union (IOU). IOU is a measure used to evaluate how well the model's predicted areas overlap with the true areas. The goal is to understand how accurately the model can predict the location of objects compared to the ground truth. Typically for image segmentation tasks, an IOU > 0.5 is acceptable, while values closer to 0.7 are sought after. }

\hlbreakable{The addition of the entropy loss in the same location of large entropy changes in a well trained network, allowed us to always achieve greater validation accuracy with the varied $\lambda$'s and nearly always improves the IOU. However, it is seen that there is a point in which we achieve approximately $\ 10\%$ increase in speed of convergence, while maintaining a small increase in accuracy. It is important to observe that once the value of $\lambda$ gets too large, it tends to overpower the base loss, forcing it to take much longer to converge. Similarly, with a $\lambda$ too small, it may take significantly more time due to varying loss changes. }

\begin{figure}[http]
    \centering
    \begin{tabular}{cccc}
        \includegraphics[width=0.208\textwidth]{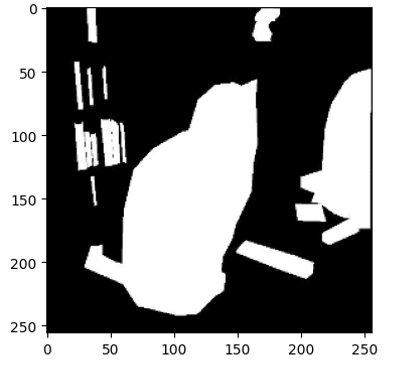} &
        \includegraphics[width=0.2\textwidth]{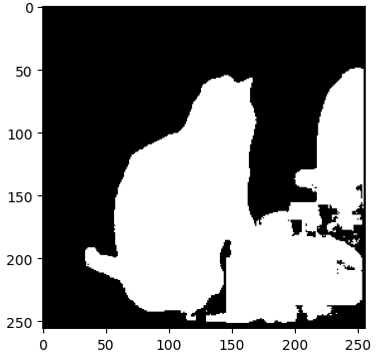} &
        \includegraphics[width=0.2\textwidth]{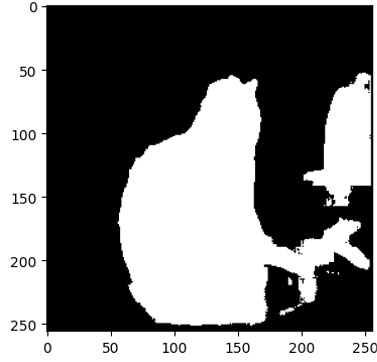} &
        \includegraphics[width=0.2\textwidth]{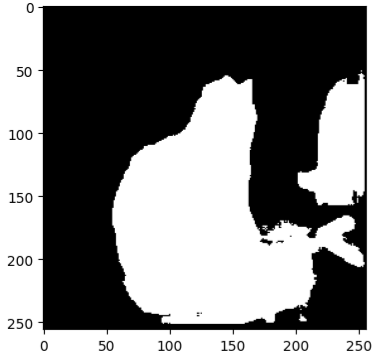} \\
        \textbf{True} & \textbf{$\lambda_2=0.000$} & \textbf{$\lambda_2=0.001$} & \textbf{$\lambda_2=0.010$} \\
    \end{tabular}
    \caption{Comparison of the binary masks produced utilizing different values of $\lambda_2$.}\label{fig:mask_comparison}
\end{figure}

\hlbreakable{To ensure the masks produced are truly higher quality, the predicted binary masks were compared to the true binary mask of the image. In Figure~\ref{fig:mask_comparison} above, the masks that are predicted when a model was trained using the entropy loss are significantly better at separating the objects in the image. For example, the true mask shows multiple small objects next to the shape of a cat, however, the model trained without the entropy loss cannot separate any of those items, while the $\lambda_2=0.001$ manages to separate them quite a bit more. }

\section{Conclusion}
This article addresses the complexity of the decision making process of neural networks by utilizing information theory. We derived novel information theoretic formulas that measure the entropy propagation through dense and 2D convolutional layers, These results provided a foundation for entropy-based loss terms that allow us to guide and regularize the neural training process.

With our loss, we first analyzed the performance of networks with these loss terms by performing experiments on image compression and image classification tasks on small benchmark datasets MNIST and CIFAR10. \hlbreakable{Specifically, the incorporation of the $L_\text{dense}(\mathcal{W})$ term in an autoencoder for image compression showed as much as a 4x increase in convergence speed and converged to nearly equally good minima.} Additionally, performance gains are demonstrated on CIFAR-10 image classification with CNNs utilizing the $L_\text{conv}(\mathcal{C})$ term, and we confirm the hypothesis that encouraging entropy preservation in early layers promotes better generalization with statistically significantly better validation accuracies.

\hlbreakable{We then scaled up to large-scale experiments on larger datasets and neural architectures. We first visualized the entropy change patterns within well-trained CNNs VGG, ResNet-50, and UNet to see how entropy changes in each convolutional layer, and discovered these exhibit an ideal information flow patterns that are consistent across datasets: VGG preserves entropy early, ResNet preserves entropy before skip connections, and UNet preserves entropy at the beginning and the end.}

\hlbreakable{Our experiments included classification of ImageNette using large VGG-16 and ResNet-50 CNNs and image segmentation of the COCO dataset with UNet. In classification, VGG-16 and ResNet-50 models with the the convolutional entropy-based loss consistently outperformed versions without the entropy loss by 2-3\%. In all cases, best results are obtained by promoting entropy as suggested by pretrained models, as in Figure~\ref{fig:block-diagram}. With image segmentation, UNet trains 10\% faster when ideal entropy patterns are encouraged with equivalent accuracy.}

This work provides strong foundational findings that allow practical information-theoretic guidance of neural networks.  It can allow theory-backed, principled construction of neural architectures, specifically in terms of depth, width, and layer structure by analyzing entropy patterns and promoting optimal patterns.

Related works have often searched for similar methods, but have been focused on the estimation of entropy. This often leads to highly complex calculations that can be difficult for many practical use cases. Our work avoids these difficult estimations of entropy, and focuses instead on the more easily computable change in entropy, localized to specific layers and channels of the network. This calculation is cheap in comparison, and can be used in larger practical models.

\hlbreakable{\subsection{Directions for Future Work}}

\hlbreakable{The methods and findings herein prompt work in several directions. Our approach provides an opportunity to improve existing MLPs and CNNs by simply adding our entropy-based loss, studying entropy flow through different neural architectures for different tasks, and for making neural networks more efficient, accurate, and interpretable.}

\hlbreakable{This paper considered only dense and convolutional neural structures with skip-connections, and the feature extraction by weight matrices and convolutional filters. Similar entropy-based analysis could be done to derive formulas for entropy propagation through additional architectural structures. Extending the entropy change results to activation functions would provide understanding of information flow through the full forward pass in MLPs and CNNs. Extensions to attention mechanisms stand to extend such analysis and guidance to transformers like state-of-the-art large language models and vision transformers. Extensions to recurrent models could find entropy change patterns within long short-term memory (LSTM) units or gated recurrent units (GRU) for use in models commonly used for temporal data.}

\hlbreakable{In addition, the paper was limited to computer vision tasks of image compression, classification, and segmentation. Our approach should trivially extend to other vision tasks like object detection or 3D rendering with neural radiance fields or related methods. Further, the same ideas can be used outside computer vision, anywhere dense or convolutional networks are used: from the dense layers of transformers to convolutions applied to Mel-spectrograms in audio signals analysis to generative models, e.g. for dense or convolutional GANs or the UNets in diffusion models.}

\hlbreakable{Lastly, analysis of information flow to extract interpretable patterns has the potential to provide numerous benefits. Locating unhelpful parts of a neural architecture can enable model compression or pruning. Applying losses to encourage diverse information flow patterns has the potential to diversify learners in an ensemble. Studies of the relationship between information flow and model decisions may lead to a deeper understanding of the decision-making processes within neural networks, which is important for human trust and can help with identifying and removing bias.}

\section*{Acknowledgement}
The authors would like to thank Olivia Raney for her invaluable editorial improvements to this article. M. Meni would also like to thank Dr. Kaleb Smith for his willingness to share his knowledge and provide thoughtful recommendations to overcome challenges. R. T. White wishes to thank the NVIDIA Applied Research Accelerator Program for providing hardware support for this effort.

\section*{Disclosure}

Opinions, interpretations, conclusions, and recommendations are those of the author(s) and are not necessarily endorsed by the U.S. Army.

%\nocite{*}
\bibliographystyle{plainnat}
\bibliography{entropy}

\begin{thebibliography}{50}
\providecommand{\natexlab}[1]{#1}
\providecommand{\url}[1]{\texttt{#1}}
\expandafter\ifx\csname urlstyle\endcsname\relax
  \providecommand{\doi}[1]{doi: #1}\else
  \providecommand{\doi}{doi: \begingroup \urlstyle{rm}\Url}\fi

\bibitem[Amanova et~al.(2024)Amanova, Martin, and Elster]{amanova2024finding}
Narbota Amanova, Jörg Martin, and Clemens Elster.
\newblock Finding the input features that reduce the entropy of a neural
  network’s prediction.
\newblock \emph{Applied Intelligence}, 54\penalty0 (2):\penalty0 1922--1936,
  2024.
\newblock ISSN 1573-7497.
\newblock \doi{10.1007/s10489-024-05277-5}.

\bibitem[Beirlant et~al.(1997)Beirlant, Dudewicz, Gy{\"o}rfi, Van~der Meulen,
  et~al.]{beirlant1997nonparametric}
Jan Beirlant, Edward~J Dudewicz, L{\'a}szl{\'o} Gy{\"o}rfi, Edward~C Van~der
  Meulen, et~al.
\newblock Nonparametric entropy estimation: An overview.
\newblock \emph{International Journal of Mathematical and Statistical
  Sciences}, 6\penalty0 (1):\penalty0 17--39, 1997.

\bibitem[Bickel and Breiman(1983)]{bickel1983sums}
Peter~J Bickel and Leo Breiman.
\newblock Sums of functions of nearest neighbor distances, moment bounds, limit
  theorems and a goodness of fit test.
\newblock \emph{The Annals of Probability}, pages 185--214, 1983.

\bibitem[Chen et~al.(2016)Chen, Duan, Houthooft, Schulman, Sutskever, and
  Abbeel]{chen_2016_infogan}
Xi~Chen, Yan Duan, Rein Houthooft, John Schulman, Ilya Sutskever, and Pieter
  Abbeel.
\newblock Info{GAN}: Interpretable representation learning by information
  maximizing generative adversarial nets.
\newblock \emph{Advances in {Neural} {Information} {Processing} {Systems}}, 29,
  2016.

\bibitem[Chopra()]{chopra_coco_unet}
Hetarth Chopra.
\newblock Coco unet.
\newblock URL
  \url{https://www.kaggle.com/code/hetarthchopra/coco-unet/notebook}.
\newblock Accessed on Feb 27, 2024.

\bibitem[Cover and Thomas(2006)]{Cover2006}
Thomas~M. Cover and Joy~A. Thomas.
\newblock \emph{Elements of Information Theory 2nd Edition (Wiley Series in
  Telecommunications and Signal Processing)}.
\newblock Wiley-Interscience, July 2006.
\newblock ISBN 0471241954.

\bibitem[Deng et~al.(2009)Deng, Dong, Socher, Li, Li, and
  Fei-Fei]{deng_imagenet_2009}
Jia Deng, Wei Dong, Richard Socher, Li-Jia Li, Kai Li, and Li~Fei-Fei.
\newblock {ImageNet}: {A} large-scale hierarchical image database.
\newblock In \emph{2009 {IEEE} {Conference} on {Computer} {Vision} and
  {Pattern} {Recognition}}, pages 248--255, June 2009.
\newblock \doi{10.1109/CVPR.2009.5206848}.
\newblock ISSN: 1063-6919.

\bibitem[Dmitriev and Tarasenko(1974)]{dmitriev1974estimation}
Yu.~G. Dmitriev and F.~P. Tarasenko.
\newblock On the estimation of functionals of the probability density and its
  derivatives.
\newblock \emph{Theory of Probability \& Its Applications}, 18\penalty0
  (3):\penalty0 628--633, 1974.
\newblock ISSN 1095-7219.
\newblock \doi{10.1137/1118083}.

\bibitem[Dong et~al.(2023)Dong, Gong, Chen, and Li]{dong2023understanding}
Yuxin Dong, Tieliang Gong, Hong Chen, and Chen Li.
\newblock Understanding the generalization ability of deep learning algorithms:
  A kernelized rényi’s entropy perspective.
\newblock In \emph{Proceedings of the Thirty-Second International Joint
  Conference on Artificial Intelligence}, IJCAI-2023, pages 3642--3650.
  International Joint Conferences on Artificial Intelligence Organization,
  2023.
\newblock \doi{10.24963/ijcai.2023/405}.

\bibitem[Dosovitskiy et~al.(2021)Dosovitskiy, Beyer, Kolesnikov, Weissenborn,
  Zhai, Unterthiner, Dehghani, Minderer, Heigold, Gelly, Uszkoreit, and
  Houlsby]{dosovitskiy_image_2021}
Alexey Dosovitskiy, Lucas Beyer, Alexander Kolesnikov, Dirk Weissenborn,
  Xiaohua Zhai, Thomas Unterthiner, Mostafa Dehghani, Matthias Minderer, Georg
  Heigold, Sylvain Gelly, Jakob Uszkoreit, and Neil Houlsby.
\newblock An image is worth 16x16 words: Transformers for image recognition at
  scale.
\newblock In \emph{9th International Conference on Learning Representations,
  {ICLR} 2021, Virtual Event, Austria, May 3-7, 2021}. OpenReview.net, 2021.

\bibitem[Erdogmus et~al.(2003)Erdogmus, Hild, and
  Principe]{erdogmus_online_2003}
D.~Erdogmus, K.E. Hild, and J.C. Principe.
\newblock Online entropy manipulation: stochastic information gradient.
\newblock \emph{IEEE Signal Processing Letters}, 10\penalty0 (8):\penalty0
  242--245, August 2003.
\newblock ISSN 1558-2361.
\newblock \doi{10.1109/LSP.2003.814400}.
\newblock Conference Name: IEEE Signal Processing Letters.

\bibitem[Finnegan and Song(2017)]{finnegan2017maximum}
Alex Finnegan and Jun~S. Song.
\newblock Maximum entropy methods for extracting the learned features of deep
  neural networks.
\newblock \emph{PLoS computational biology}, 13\penalty0 (10):\penalty0
  e1005836, 2017.
\newblock ISSN 1553-7358.
\newblock \doi{10.1371/journal.pcbi.1005836}.

\bibitem[Gabri{\'e} et~al.(2018)Gabri{\'e}, Manoel, Luneau, Macris, Krzakala,
  Zdeborov{\'a}, et~al.]{gabrie2018entropy}
Marylou Gabri{\'e}, Andre Manoel, Cl{\'e}ment Luneau, Nicolas Macris, Florent
  Krzakala, Lenka Zdeborov{\'a}, et~al.
\newblock Entropy and mutual information in models of deep neural networks.
\newblock \emph{Advances in {Neural} {Information} {Processing} {Systems}}, 31,
  2018.

\bibitem[Gajowniczek et~al.(2020)Gajowniczek, Liang, Friedman, Ząbkowski, and
  Van~den Broeck]{Gajowniczek_GEL_2020}
Krzysztof Gajowniczek, Yitao Liang, Tal Friedman, Tomasz Ząbkowski, and Guy
  Van~den Broeck.
\newblock Semantic and generalized entropy loss functions for semi-supervised
  deep learning.
\newblock \emph{Entropy}, 22\penalty0 (3), 2020.
\newblock ISSN 1099-4300.
\newblock \doi{10.3390/e22030334}.

\bibitem[Goodfellow et~al.(2014)Goodfellow, Pouget-Abadie, Mirza, Xu,
  Warde-Farley, Ozair, Courville, and Bengio]{goodfellow_generative_2014}
Ian Goodfellow, Jean Pouget-Abadie, Mehdi Mirza, Bing Xu, David Warde-Farley,
  Sherjil Ozair, Aaron Courville, and Yoshua Bengio.
\newblock Generative {Adversarial} {Nets}.
\newblock In \emph{Advances in {Neural} {Information} {Processing} {Systems}},
  volume~27. Curran Associates, Inc., 2014.

\bibitem[Györfi and van~der Meulen(1987)]{gyorfi1987density}
László Györfi and Edward~C. van~der Meulen.
\newblock Density-free convergence properties of various estimators of entropy.
\newblock \emph{Computational Statistics \& Data Analysis}, 5\penalty0
  (4):\penalty0 425--436, 1987.
\newblock ISSN 0167-9473.
\newblock \doi{10.1016/0167-9473(87)90065-x}.

\bibitem[Haarnoja et~al.(2018)Haarnoja, Zhou, Abbeel, and
  Levine]{haarnoja2018soft}
Tuomas Haarnoja, Aurick Zhou, Pieter Abbeel, and Sergey Levine.
\newblock Soft actor-critic: Off-policy maximum entropy deep reinforcement
  learning with a stochastic actor.
\newblock In \emph{International conference on machine learning}, pages
  1861--1870. PMLR, 2018.

\bibitem[Han et~al.(2023)Han, Zhou, Lü, Zhu, and Gong]{han2023entropy}
Shuai Han, Wenbo Zhou, Shuai Lü, Sheng Zhu, and Xiaoyu Gong.
\newblock Entropy regularization methods for parameter space exploration.
\newblock \emph{Information Sciences}, 622:\penalty0 476--489, 2023.
\newblock ISSN 0020-0255.
\newblock \doi{10.1016/j.ins.2022.11.099}.

\bibitem[Hao et~al.(2022)Hao, Zhang, Shi, and Li]{hao2022entropy}
Dong Hao, Dongcheng Zhang, Qi~Shi, and Kai Li.
\newblock Entropy regularized actor-critic based multi-agent deep reinforcement
  learning for stochastic games.
\newblock \emph{Information Sciences}, 617:\penalty0 17--40, 2022.
\newblock ISSN 0020-0255.
\newblock \doi{10.1016/j.ins.2022.10.022}.

\bibitem[Hayashi et~al.(2023)Hayashi, Cimr, Fujita, and
  Cimler]{hayashi2023image}
Toshitaka Hayashi, Dalibor Cimr, Hamido Fujita, and Richard Cimler.
\newblock Image entropy equalization: A novel preprocessing technique for image
  recognition tasks.
\newblock \emph{Information Sciences}, 647:\penalty0 119539, 2023.
\newblock ISSN 0020-0255.
\newblock \doi{10.1016/j.ins.2023.119539}.

\bibitem[He et~al.(2016)He, Zhang, Ren, and Sun]{he2016deep}
Kaiming He, Xiangyu Zhang, Shaoqing Ren, and Jian Sun.
\newblock Deep residual learning for image recognition.
\newblock In \emph{2016 IEEE Conference on Computer Vision and Pattern
  Recognition (CVPR)}, pages 770--778. IEEE, 2016.
\newblock \doi{10.1109/cvpr.2016.90}.

\bibitem[Hjelm et~al.(2019)Hjelm, Fedorov, Lavoie-Marchildon, Grewal, Bachman,
  Trischler, and Bengio]{hjelm_learning_2019}
R.~Devon Hjelm, Alex Fedorov, Samuel Lavoie-Marchildon, Karan Grewal, Philip
  Bachman, Adam Trischler, and Yoshua Bengio.
\newblock Learning deep representations by mutual information estimation and
  maximization.
\newblock In \emph{7th {International} {Conference} on {Learning}
  {Representations}, {ICLR} 2019, {New} {Orleans}, {LA}, {USA}, {May} 6-9,
  2019}. OpenReview.net, 2019.

\bibitem[Ho et~al.(2020)Ho, Jain, and Abbeel]{ho_denoising_2020}
Jonathan Ho, Ajay Jain, and Pieter Abbeel.
\newblock Denoising {Diffusion} {Probabilistic} {Models}.
\newblock In \emph{Advances in {Neural} {Information} {Processing} {Systems}},
  volume~33, pages 6840--6851. Curran Associates, Inc., 2020.

\bibitem[Howard and Gugger(2019)]{Howard_Imagenette_2019}
Jeremy Howard and Sylvain Gugger.
\newblock Fastai: A layered api for deep learning.
\newblock \emph{Information}, 11\penalty0 (2):\penalty0 108, March 2019.
\newblock ISSN 2078-2489.
\newblock \doi{10.3390/info11020108}.

\bibitem[Kingma and Welling(2013)]{kingma_auto-encoding_2013}
Diederik~P Kingma and Max Welling.
\newblock Auto-encoding variational bayes.
\newblock \emph{arXiv}, 2013.
\newblock \doi{10.48550/ARXIV.1312.6114}.

\bibitem[Krizhevsky(2009)]{Krizhevsky2009LearningML}
Alex Krizhevsky.
\newblock Learning multiple layers of features from tiny images.
\newblock 2009.

\bibitem[Krizhevsky et~al.(2012)Krizhevsky, Sutskever, and
  Hinton]{krizhevsky_imagenet_2012}
Alex Krizhevsky, Ilya Sutskever, and Geoffrey~E Hinton.
\newblock {ImageNet} {Classification} with {Deep} {Convolutional} {Neural}
  {Networks}.
\newblock In F.~Pereira, C.~J. Burges, L.~Bottou, and K.~Q. Weinberger,
  editors, \emph{Advances in {Neural} {Information} {Processing} {Systems}},
  volume~25. Curran Associates, Inc., 2012.

\bibitem[LeCun et~al.(1998)LeCun, Bottou, Bengio, and
  Haffner]{lecun_gradient-based_1998}
Y.~LeCun, L.~Bottou, Y.~Bengio, and P.~Haffner.
\newblock Gradient-based learning applied to document recognition.
\newblock \emph{Proceedings of the IEEE}, 86\penalty0 (11):\penalty0
  2278--2324, November 1998.
\newblock ISSN 1558-2256.
\newblock \doi{10.1109/5.726791}.
\newblock Conference Name: Proceedings of the IEEE.

\bibitem[LeCun et~al.(2010)LeCun, Cortes, and Burges]{lecun2010mnist}
Yann LeCun, Corinna Cortes, and CJ~Burges.
\newblock Mnist handwritten digit database.
\newblock \emph{ATT Labs [Online]. Available:
  http://yann.lecun.com/exdb/mnist}, 2, 2010.

\bibitem[Lhermitte et~al.(2022)Lhermitte, Hilal, Furlong, O’Brien, and
  Humeau-Heurtier]{lhermitte2022deep}
Emma Lhermitte, Mirvana Hilal, Ryan Furlong, Vincent O’Brien, and Anne
  Humeau-Heurtier.
\newblock Deep learning and entropy-based texture features for color image
  classification.
\newblock \emph{Entropy}, 24\penalty0 (11):\penalty0 1577, 2022.
\newblock ISSN 1099-4300.
\newblock \doi{10.3390/e24111577}.

\bibitem[Lin et~al.(2014)Lin, Maire, Belongie, Hays, Perona, Ramanan, Dollár,
  and Zitnick]{lin_microsoft_2014}
Tsung-Yi Lin, Michael Maire, Serge Belongie, James Hays, Pietro Perona, Deva
  Ramanan, Piotr Dollár, and C.~Lawrence Zitnick.
\newblock Microsoft {COCO}: {Common} {Objects} in {Context}.
\newblock In David Fleet, Tomas Pajdla, Bernt Schiele, and Tinne Tuytelaars,
  editors, \emph{Computer {Vision} – {ECCV} 2014}, Lecture {Notes} in
  {Computer} {Science}, pages 740--755, Cham, 2014. Springer International
  Publishing.
\newblock ISBN 978-3-319-10602-1.
\newblock \doi{10.1007/978-3-319-10602-1_48}.

\bibitem[Linsker(1988)]{linsker_application_1988}
Ralph Linsker.
\newblock An {Application} of the {Principle} of {Maximum} {Information}
  {Preservation} to {Linear} {Systems}.
\newblock In \emph{Advances in {Neural} {Information} {Processing} {Systems}},
  volume~1. Morgan-Kaufmann, 1988.

\bibitem[Liu et~al.(2020)Liu, Shi, Miao, Jin, and Zhou]{Liu_RDE_2020}
Pingping Liu, Lida Shi, Zhuang Miao, Baixin Jin, and Qiuzhan Zhou.
\newblock Relative distribution entropy loss function in cnn image retrieval.
\newblock \emph{Entropy}, 22\penalty0 (3), 2020.
\newblock ISSN 1099-4300.
\newblock \doi{10.3390/e22030321}.

\bibitem[Mnih et~al.(2016)Mnih, Badia, Mirza, Graves, Lillicrap, Harley,
  Silver, and Kavukcuoglu]{mnih2016asynchronous}
Volodymyr Mnih, Adria~Puigdomenech Badia, Mehdi Mirza, Alex Graves, Timothy
  Lillicrap, Tim Harley, David Silver, and Koray Kavukcuoglu.
\newblock Asynchronous methods for deep reinforcement learning.
\newblock In Maria~Florina Balcan and Kilian~Q. Weinberger, editors,
  \emph{Proceedings of The 33rd International Conference on Machine Learning},
  volume~48 of \emph{Proceedings of Machine Learning Research}, pages
  1928--1937, New York, New York, USA, 20--22 Jun 2016. PMLR.

\bibitem[Ouyang et~al.(2022)Ouyang, Wu, Jiang, Almeida, Wainwright, Mishkin,
  Zhang, Agarwal, Slama, Ray, Schulman, Hilton, Kelton, Miller, Simens, Askell,
  Welinder, Christiano, Leike, and Lowe]{ouyang_training_2022}
Long Ouyang, Jeffrey Wu, Xu~Jiang, Diogo Almeida, Carroll Wainwright, Pamela
  Mishkin, Chong Zhang, Sandhini Agarwal, Katarina Slama, Alex Ray, John
  Schulman, Jacob Hilton, Fraser Kelton, Luke Miller, Maddie Simens, Amanda
  Askell, Peter Welinder, Paul~F. Christiano, Jan Leike, and Ryan Lowe.
\newblock Training language models to follow instructions with human feedback.
\newblock \emph{Advances in Neural Information Processing Systems},
  35:\penalty0 27730--27744, December 2022.

\bibitem[Principe et~al.(2000)Principe, Xu, Zhao, and
  Fisher]{principe_learning_2000}
Jose~C. Principe, Dongxin Xu, Qun Zhao, and John~W. Fisher.
\newblock Learning from examples with information theoretic criteria.
\newblock \emph{Journal of VLSI signal processing systems for signal, image and
  video technology}, 26\penalty0 (1):\penalty0 61--77, August 2000.
\newblock ISSN 0922-5773.
\newblock \doi{10.1023/A:1008143417156}.

\bibitem[Qi and Majda(2020)]{QM_kldiv_2020}
Di~Qi and Andrew~J. Majda.
\newblock Using machine learning to predict extreme events in complex systems.
\newblock \emph{Proceedings of the National Academy of Sciences}, 117\penalty0
  (1):\penalty0 52--59, 2020.
\newblock \doi{10.1073/pnas.1917285117}.

\bibitem[Radford et~al.(2018)Radford, Narasimhan, Salimans, and
  Sutskever]{radford_improving_2018}
Alec Radford, Karthik Narasimhan, Tim Salimans, and Ilya Sutskever.
\newblock Improving {Language} {Understanding} by {Generative}
  {Pre}-{Training}.
\newblock 2018.

\bibitem[Ronneberger et~al.(2015)Ronneberger, Fischer, and
  Brox]{ronneberger2015Unet}
Olaf Ronneberger, Philipp Fischer, and Thomas Brox.
\newblock \emph{U-Net: Convolutional Networks for Biomedical Image
  Segmentation}, pages 234--241.
\newblock Springer International Publishing, Cham, 2015.
\newblock ISBN 9783319245744.
\newblock \doi{10.1007/978-3-319-24574-4_28}.

\bibitem[Rudy and Sapsis(2023)]{rudy2023output}
Samuel~H. Rudy and Themistoklis~P. Sapsis.
\newblock Output-weighted and relative entropy loss functions for deep learning
  precursors of extreme events.
\newblock \emph{Physica D: Nonlinear Phenomena}, 443:\penalty0 133570, 2023.
\newblock ISSN 0167-2789.
\newblock \doi{10.1016/j.physd.2022.133570}.

\bibitem[Simonyan and Zisserman(2015)]{simonyan_very_2015}
Karen Simonyan and Andrew Zisserman.
\newblock Very deep convolutional networks for large-scale image recognition.
\newblock In Yoshua Bengio and Yann LeCun, editors, \emph{3rd International
  Conference on Learning Representations, {ICLR} 2015, San Diego, CA, USA, May
  7-9, 2015, Conference Track Proceedings}, 2015.

\bibitem[Stiennon et~al.(2020)Stiennon, Ouyang, Wu, Ziegler, Lowe, Voss,
  Radford, Amodei, and Christiano]{stiennon_learning_2020}
Nisan Stiennon, Long Ouyang, Jeffrey Wu, Daniel Ziegler, Ryan Lowe, Chelsea
  Voss, Alec Radford, Dario Amodei, and Paul~F Christiano.
\newblock Learning to summarize with human feedback.
\newblock In \emph{Advances in {Neural} {Information} {Processing} {Systems}},
  volume~33, pages 3008--3021. Curran Associates, Inc., 2020.

\bibitem[Tarasenko(1968)]{tarasenko1968evaluation}
F.~P. Tarasenko.
\newblock On the evaluation of an unknown probability density function, the
  direct estimation of the entropy from independent observations of a
  continuous random variable, and the distribution-free entropy test of
  goodness-of-fit.
\newblock \emph{Proceedings of the IEEE}, 56\penalty0 (11):\penalty0
  2052--2053, 1968.
\newblock \doi{10.1109/PROC.1968.6784}.

\bibitem[Tishby et~al.(1999)Tishby, Pereira, and
  Bialek]{tishby_information_1999}
Naftali Tishby, Fernando~C. Pereira, and William Bialek.
\newblock The information bottleneck method.
\newblock In \emph{Proceedings of the 37th {Annual} {Allerton} {Conference} on
  {Communication} {Control} and {Computing}}, pages 368--377, 1999.

\bibitem[Vaswani et~al.(2017)Vaswani, Shazeer, Parmar, Uszkoreit, Jones, Gomez,
  Kaiser, and Polosukhin]{vaswani2017attention}
Ashish Vaswani, Noam Shazeer, Niki Parmar, Jakob Uszkoreit, Llion Jones,
  Aidan~N Gomez, {\L}ukasz Kaiser, and Illia Polosukhin.
\newblock Attention is all you need.
\newblock In \emph{Advances in {Neural} {Information} {Processing} {Systems}},
  volume~30, 2017.

\bibitem[Vincent et~al.(2010)Vincent, Larochelle, Lajoie, Bengio, and
  Manzagol]{vincent_stacked_2010}
Pascal Vincent, Hugo Larochelle, Isabelle Lajoie, Yoshua Bengio, and
  Pierre-Antoine Manzagol.
\newblock Stacked {Denoising} {Autoencoders}: {Learning} {Useful}
  {Representations} in a {Deep} {Network} with a {Local} {Denoising}
  {Criterion}.
\newblock \emph{Journal of Machine Learning Research}, 11\penalty0
  (110):\penalty0 3371--3408, 2010.
\newblock ISSN 1533-7928.

\bibitem[Xu et~al.(2023)Xu, Li, Lin, Gao, Guo, Lü, and Zhang]{xu_q-detr_2023}
Sheng Xu, Yanjing Li, Mingbao Lin, Peng Gao, Guodong Guo, Jinhu Lü, and
  Baochang Zhang.
\newblock Q-{DETR}: {An} {Efficient} {Low}-{Bit} {Quantized} {Detection}
  {Transformer}.
\newblock In \emph{Proceedings of the IEEE/CVF Conference on Computer Vision
  and Pattern Recognition}, pages 3842--3851, 2023.

\bibitem[Yu et~al.(2021{\natexlab{a}})Yu, Sanchez~Giraldo, and
  Principe]{yu_information-theoretic_2021}
Shujian Yu, Luis Sanchez~Giraldo, and Jose Principe.
\newblock Information-theoretic methods in deep neural networks: Recent
  advances and emerging opportunities.
\newblock In \emph{Proceedings of the Thirtieth International Joint Conference
  on Artificial Intelligence}, volume~5 of \emph{IJCAI-2021}, pages 4669--4678.
  International Joint Conferences on Artificial Intelligence Organization,
  August 2021{\natexlab{a}}.
\newblock \doi{10.24963/ijcai.2021/633}.
\newblock ISSN: 1045-0823.

\bibitem[Yu et~al.(2021{\natexlab{b}})Yu, Yu, and Príncipe]{yu_deep_2021}
Xi~Yu, Shujian Yu, and José~C. Príncipe.
\newblock Deep {Deterministic} {Information} {Bottleneck} with {Matrix}-{Based}
  {Entropy} {Functional}.
\newblock In \emph{{ICASSP} 2021 - 2021 {IEEE} {International} {Conference} on
  {Acoustics}, {Speech} and {Signal} {Processing} ({ICASSP})}, pages
  3160--3164, June 2021{\natexlab{b}}.
\newblock \doi{10.1109/ICASSP39728.2021.9414151}.
\newblock ISSN: 2379-190X.

\bibitem[Özdenizci and Erdoğmuş(2021)]{ozdenizci2021stochastic}
Ozan Özdenizci and Deniz Erdoğmuş.
\newblock Stochastic mutual information gradient estimation for dimensionality
  reduction networks.
\newblock \emph{Information sciences}, 570:\penalty0 298--305, 2021.
\newblock ISSN 0020-0255.
\newblock \doi{10.1016/j.ins.2021.04.066}.

\end{thebibliography}

\end{document}